\documentclass[times,review,10pt]{elsarticle}
\usepackage{amsmath}
\usepackage{amsthm}
\newtheorem{proposition}{Proposition}

\newtheorem{remark}{Remark}
\usepackage{graphicx}
\graphicspath{{figs/}}
\usepackage{algorithm}
\usepackage{algorithmic}
\usepackage{array}
\usepackage{booktabs}
\usepackage{caption}
\usepackage{enumitem}
\usepackage{float}
\usepackage{multicol}
\usepackage{multirow}
\usepackage{nicefrac}
\usepackage{microtype}
\usepackage{tcolorbox}
\usepackage{subcaption}
\usepackage{wrapfig}
\usepackage{xcolor}
\usepackage[table]{xcolor}
\usepackage[colorlinks=true, linkcolor=blue, citecolor=blue, urlcolor=blue]{hyperref}
\usepackage{cleveref}
\usepackage{colortbl}
\definecolor{defaultgray}{gray}{0.85}

\journal{Pattern Recognition}

\begin{document}

\begin{frontmatter}



\title{Geometry-Aware Preference Optimization for Text-to-Image Diffusion Models}




\author[sysu]{\corref{cor1}{\footnotesize Lei Wang}}
\ead{likwang259@gmail.com}
\author[sysu]{{\footnotesize Zhen Wang}}
\author[alibaba]{{\footnotesize Yuexiang Xie}}
\author[alibaba]{{\footnotesize Yaliang Li}}

\cortext[cor1]{Corresponding author}

\affiliation[sysu]{organization={Sun Yat-Sen University},
            city={Guangzhou},
            country={China}}
\affiliation[alibaba]{organization={Alibaba Group},
            city={Hangzhou},
            country={China}}

\begin{abstract}
Preference alignment has become a standard practice for text-to-image diffusion models. Direct Preference Optimization (DPO) simplifies this process by eliminating explicit reward modeling. Its diffusion variant, Diffusion-DPO, has become a widely adopted baseline. Diffusion-DPO essentially encourages the likelihood of preferred samples while suppressing dispreferred ones. In this paper, we revisit DPO-style alignment methods for diffusion models from the perspective of the manifold hypothesis. Under this view, natural images concentrate near a low-dimensional manifold embedded in the high-dimensional ambient space, whereas DPO directly optimizes preference distributions in the full space without accounting for this geometric structure. This creates a mismatch in the optimization dynamics: it suppresses geometry-preserving tangential updates, while insufficiently restricting hazardous normal-direction updates. This mismatch gradually degrades image quality and diversity. To address this issue, we propose \textbf{A}nisotropic Geometry-Aware \textbf{P}reference \textbf{O}ptimization (\textbf{APO}), which replaces the uniform Euclidean treatment of prediction errors with a geometry-aware anisotropic metric derived from the reference model. Concretely, APO adaptively strengthens regularization in directions where the reference denoising function is highly sensitive, while relaxing constraints in directions that permit safe semantic adjustment. This recalibrates preference optimization according to the local manifold geometry, and maintains the original manifold structure. Experiments show that APO achieves strong performance and an average win rate exceeding \textbf{60\%} against various existing alignment methods across diverse benchmarks. It requires significantly fewer training steps than prior methods, and preserves generation diversity throughout training.
\end{abstract}




\end{frontmatter}

\section{Introduction}
Text-to-image (T2I) diffusion models have achieved remarkable progress in recent years, enabling high-fidelity and semantically controllable image synthesis across diverse domains~\cite{LD,sdxl,SD3-Medium,labs2025flux}. Inspired by the success of reinforcement learning for alignment in large language models, a growing number of post-training strategies have recently been introduced for diffusion models to improve human preference alignment, aesthetic quality, and prompt consistency~\cite{DDPO,DPOK,AlignProp,DRaFT,FlowGRPO}. Most existing approaches rely on training a reward model to learn preference signals and guide policy optimization. Diffusion-DPO~\cite{DPO}, which adapts Direct Preference Optimization~\cite{llmdpo} to diffusion models, has emerged as a practical alignment approach. It directly optimizes pairwise preferences through implicit likelihood comparisons while avoiding explicit reward model training, substantially simplifying the alignment pipeline. Following Diffusion-DPO, numerous variants have been proposed to improve aspects such as reward modeling, sampling efficiency, optimization stability, and model adaptability in diffusion alignment\cite{KTO, InPO, SDPO, LPO}.

Despite these advances, we find that DPO-based alignment of diffusion models still suffers from a previously overlooked geometric failure mode: off-manifold drift. The manifold hypothesis posits that natural image distributions concentrate near a low-dimensional manifold embedded in the ambient latent denoising space~\cite{bengio2013representation,fefferman2016testing}. In diffusion models, preference optimization is performed in the noisy latent denoising space conditioned on the text prompt, where the local geometry evolves across diffusion timesteps. However, although preference signals are ultimately defined over natural images, optimization is still performed in the full high-dimensional latent denoising space without considering the intrinsic geometry of the underlying data manifold. In particular, existing DPO-style objectives~\cite{DPO} optimize denoising prediction errors under an isotropic Euclidean metric, implicitly treating all perturbation directions in the latent space equally. This leads to two consequences:
\begin{itemize}[leftmargin=1.5em, topsep=0pt, partopsep=0pt, parsep=0pt]
    \item Updates along manifold-normal directions can easily push samples away from the natural image manifold, causing semantic distortion and quality degradation. Without explicit suppression of these geometry-sensitive directions, off-manifold drift accumulates during alignment.
    \item Updates along manifold-tangent directions are generally safer and often capture semantically meaningful variations. However, they are constrained by the same uniform regularization strength, unnecessarily limiting flexible semantic adjustments in preference optimization.
\end{itemize}
While some prior works have incorporated manifold-aware geometric constraints into generative dynamics, such considerations remain largely unexplored in post-training preference alignment for pretrained diffusion models.

To address this issue, we propose \textbf{Anisotropic Geometry-Aware Preference Optimization (APO)}. APO constructs a Jacobian-induced anisotropic metric from the local score-field geometry of the frozen reference model, adaptively suppressing updates along geometry-sensitive normal directions while allowing more flexible exploration along manifold-aligned tangent directions. To make the method computationally practical, we further derive an efficient forward-only approximation based on Hutchinson estimation and finite differences, eliminating the expensive reference-model backward pass while preserving the geometric regularization effect.

Our key observations and technical contributions are summarized as follows:
\begin{itemize}[leftmargin=1.5em, topsep=0pt, partopsep=0pt, parsep=0pt]
    \item We identify off-manifold drift as a fundamental geometric failure mode in DPO-based diffusion alignment, and provide theoretical and empirical analyses connecting this phenomenon to isotropic optimization in high-dimensional denoising space.
    \item We propose APO, a Jacobian-induced anisotropic preference optimization framework together with an efficient forward-only Hutchinson finite-difference approximation, enabling scalable geometry-aware diffusion alignment without reference-model backpropagation.
    \item Extensive experiments on SD1.5 and SDXL across multiple benchmarks demonstrate that APO consistently improves alignment quality, image realism, semantic preservation, and generation diversity while maintaining strong training efficiency.
\end{itemize}

\section{Related Works}

\paragraph{Diffusion Models for Visual Generation} Diffusion models have emerged as the dominant paradigm for image synthesis due to their training stability and high sample quality \cite{sohl2015deep, ddpm, SM, FM}. The introduction of latent diffusion substantially reduced computational costs by operating the denoising process in a compressed representation space, enabling high-resolution text-to-image generation at scale \cite{LD}. This architecture catalyzed a series of large-scale open-weight models, including SDXL \cite{sdxl}, SD3 \cite{SD3-Medium}, and the FLUX series \cite{labs2025flux}, which form the backbone of most contemporary visual generation research. However, diffusion models are often trained on large-scale general-purpose datasets with varying quality, which tends to introduce biases that deviate from human preferences. Therefore, aligning diffusion models with human preferences is a highly necessary practice.

\paragraph{Reward-Based Reinforcement Learning for Diffusion} Early post-training efforts adopt supervised fine-tuning (SFT) on curated high-quality image-text pairs \cite{SFT}. While effective, SFT inherently treats all samples equally and does not capture the comparative nature of human preference. Reinforcement learning from human feedback (RLHF) \cite{christiano2017deep, ouyang2022training} addresses this by first training a proxy reward model on human annotations and then optimizing the diffusion policy with policy gradient methods such as DDPO \cite{DDPO} and DPOK \cite{DPOK}, or by directly backpropagating differentiable rewards as in DRaFT \cite{DRaFT} and AlignProp \cite{AlignProp}. These approaches, however, are prone to reward overoptimization. To bypass explicit reward learning, Diffusion-DPO \cite{DPO} adapts Direct Preference Optimization to diffusion models, enabling offline fine-tuning directly on pairwise preference data without a separate reward model. Subsequently, a series of methods have been proposed to address the weaknesses of Diffusion-DPO in various aspects, including sampling efficiency, reward estimation, model adaptability, and reward diversity \cite{DSPO, KTO, SPO, InPO, MaPO, LPO, ViPO, bai2026thresholdguidedoptimizationvisualgenerative}. In addition, online reinforcement learning methods such as Flow-GRPO \cite{FlowGRPO} and DanceGRPO \cite{danceGRPO} have become an increasingly popular paradigm. Some of these approaches seek to unify tasks spanning text-to-image, text-to-video, and image-to-video generation within a single framework, while others focus on improving computational efficiency \cite{DiffusionNFT, garbacea2025hyperalign, cheng2025diffusion, tang2026v, lee2026reward}. While existing alignment methods primarily modify the reward formulation, optimization strategy, or sampling efficiency, our work instead revisits the geometric assumptions underlying the optimization metric itself.

\paragraph{Manifold Learning for Diffusion} The manifold hypothesis posits that high-dimensional natural data concentrate near a low-dimensional submanifold of the ambient space \cite{bengio2013representation, fefferman2016testing}. In the context of diffusion models, this connection is theoretically well-grounded: the score function is intimately related to the geometry of the data distribution and approximates the normal bundle of the data manifold \cite{SM, pidstrigach2022score, stanczuk2024diffusion, li2025scores}, a structural prior that has motivated a rich line of research connecting manifold learning with generative modeling \cite{huang2022riemannian, liu2025improving, JiT, li2025scores, zhang2026diffusion}. Furthermore, manifold awareness has proven essential for improving inference-time quality, as demonstrated by CFG++ \cite{Cfg++}, TCFG \cite{TCFG2}, and Be Tangential to Manifold \cite{saito2025tangential}. However, in the domain of preference alignment for diffusion models based on standard Euclidean diffusion parameterizations, manifold learning theory has yet to be considered.

\section{Background and Motivation}
In this section, we first review diffusion models and the formulation of Diffusion-DPO. We then revisit the DPO objective from a geometric perspective and show that its isotropic optimization mechanism can induce harmful off-manifold drift during preference alignment.

\subsection{Diffusion Models}
Denoising Diffusion Probabilistic Models (DDPM) \cite{ddpm} define a forward process that gradually transforms a clean sample $\mathbf{x}_0 \sim p_{\text{data}}$ into Gaussian noise over $T$ steps. Following the variance-preserving (VP) schedule, where $\alpha_t = 1 - \beta_t$, $\bar{\alpha}_t = \prod_{s=1}^t \alpha_s$, and $\sigma_t^2 = 1 - \bar{\alpha}_t$, the noisy latent $\mathbf{x}_t$ at any step $t \in [0, T]$ is:
\begin{equation}\label{eq:1}
  \mathbf{x}_t = \sqrt{\bar{\alpha}_t} \mathbf{x}_0 + \sqrt{1 - \bar{\alpha}_t} \epsilon, \quad \epsilon \sim \mathcal{N}(0, I).
\end{equation}
To generate a sample, the model learns a reverse transition distribution $p_\theta(\mathbf{x}_{t-1} | \mathbf{x}_t)$ to approximate the intractable reverse conditional of the forward diffusion process. In the standard DDPM parameterization, this transition is modeled as a Gaussian distribution:
\begin{equation}\label{eq:2}
p_\theta(\mathbf{x}_{t-1}| \mathbf{x}_t) = \mathcal{N} (\mathbf{x}_{t-1}; \mu_\theta(\mathbf{x}_t, t), \sigma_t^2 \mathbf{I});
\end{equation}
where the mean $\mu_\theta(\mathbf{x_t}, t) = \frac{1}{\sqrt{\alpha}_t} \left( \mathbf{x}_t - \frac{\beta_t}{\sqrt{1 - \bar{\alpha}_t}} \epsilon_\theta(\mathbf{x}_t, t) \right)$ is typically parameterized by a neural network that predicts the noise added to $\mathbf{x}_0$. The reverse process aims to reconstruct $\mathbf{x}_0$ by learning a denoising network $\epsilon_\theta(\mathbf{x}_t, t)$ to estimate the added noise $\epsilon$. The training objective is simplified as:
\begin{equation}\label{eq:3}
\mathcal{L}_{\mathrm{DM}}(\theta) = \mathbb{E}_{t \sim U(0,T), \mathbf{x}_0, \epsilon, \mathbf{x}_t \sim q(\mathbf{x}_t \mid \mathbf{x}_0)} \left[ \omega(\gamma_t) \| \epsilon_\theta(\mathbf{x}_t, t) - \epsilon \|^2 \right]
\end{equation}
where $\gamma_t$ is a time-dependent weighting factor. We adopt the standard DDPM formulation here for notation simplicity, and most modern diffusion and flow-based generative models share the same core denoising principle, differing mainly in the prediction target (e.g., $\epsilon$-, $x_0$-, or $v$-prediction), noise parameterization, and timestep weighting strategies.

\subsection{Direct Preference Optimization for Diffusion Models}\label{sec:preference_optimization}
RLHF aligns generative models by maximizing a reward function $r(\mathbf{x}_0, c)$ subject to a Kullback-Leibler (KL) divergence constraint \cite{DPOK}. For diffusion models, this objective is implemented through the denoising trajectory $p_\theta(\mathbf{x}_{0:T}|\mathbf{c})$:
\begin{equation}\label{eq:4}
\max_{\theta} \mathbb{E}_{\textbf{c} \sim \mathcal{D}, \mathbf{x}_0 \sim p_{\theta}(\mathbf{x}_0|\textbf{c})} \left[ r(\mathbf{x}_0, \textbf{c}) \right] - \beta \mathbb{D}_{\text{KL}} \left( p_{\theta}(\mathbf{x}_0 | \textbf{c}) || p_{\text{ref}}(\mathbf{x}_0 | \textbf{c}) \right);
\end{equation}
where $p_{\text{ref}}$ is the pre-trained reference model. While RLHF typically requires training an explicit reward model and employing policy gradient methods, Diffusion-DPO reparameterizes the reward through the optimal policy \cite{DPO}:
\begin{equation}\label{eq:5}
r(\mathbf{x}_0,\textbf{c}) = \beta \log \frac{p_{\theta}^*(\mathbf{x}_0|\textbf{c})}{p_{\text{ref}}(\mathbf{x}_0|\textbf{c})} + \beta \log Z(\textbf{c}).
\end{equation}
By substituting Eq.~\eqref{eq:5} into the Bradley-Terry model $\mathcal{L}_{\text{BT}}(\phi) := -\mathbb{E}_{(\textbf{x}_0^w, \textbf{x}_0^l, \textbf{c}) \sim \mathcal{D}}[\log \sigma(r_\phi(\textbf{x}_0^w, \textbf{c}) - r_\phi(\textbf{x}_0^l, \textbf{c}))]$, the partition function $Z(c)$ cancels out, allowing for direct optimization on preference pairs $(\mathbf{x}_0^w, \mathbf{x}_0^l)$. To adapt DPO to the multi-step nature of diffusion processes, Diffusion-DPO leverages the chain rule of KL divergence to decompose the trajectory-level objective into step-wise denoising objectives, yielding the following loss:
\begin{equation}\label{eq:6}
\begin{aligned}
\mathcal{L}(\theta) \le -\mathbb{E}_{(\textbf{c},\mathbf{x}_0^k) \sim \mathcal{D},\mathbf{x}_t^k \sim q(\mathbf{x}_t^k | \mathbf{x}_0^k)} \log \sigma ( 
\beta T [ \log \frac{p_\theta(\mathbf{x}_{t-1}^w | \mathbf{x}_t^w, \textbf{c})}{p_{\text{ref}}(\mathbf{x}_{t-1}^w | \mathbf{x}_t^w, \textbf{c})} 
- \log \frac{p_\theta(\mathbf{x}_{t-1}^l | \mathbf{x}_t^l, \textbf{c})}{p_{\text{ref}}(\mathbf{x}_{t-1}^l | \mathbf{x}_t^l, \textbf{c})} ]).
\end{aligned}
\end{equation}
For notational simplicity, we omit the conditioning prompt $\textbf{c}$ in subsequent expressions. For each sample $k\in\{w,l\}$, because the reverse transitions are Gaussian with a shared variance $\sigma_t^2$, the per-step log-ratio evaluates to
\begin{equation}\label{eq:at-iso}
\mathbb{E}_{x_{t-1}^k\sim q(\cdot|x_t^k,x_0^k)}\!\,[\, \log\frac{p_\theta(x_{t-1}^k|x_t^k,c)}{p_{\mathrm{ref}}(x_{t-1}^k|x_t^k,c)} ] = \frac{1}{2\sigma_t^2}( \|\mu_q^k - \mu_{\mathrm{ref}}^k\|^2 - \|\mu_q^k - \mu_\theta^k\|^2 ),
\end{equation}
where $\mu_q^k$, $\mu_\theta^k$, $\mu_{\mathrm{ref}}^k$ are the means of $q(\cdot|x_t^k,x_0^k)$, $p_\theta$, and $p_{\mathrm{ref}}$, respectively.
Using the DDPM mean parameterization (Eq.~\eqref{eq:2}), the mean differences become proportional to noise-prediction errors, ultimately simplifying Eq.~\eqref{eq:6} to the practical MSE-based contrastive loss:
\begin{equation}\label{eq:7}
\begin{aligned}
\mathcal{L}(\theta) = & -\mathbb{E}_{\mathbf{x}_0^k \sim D, t \sim U(0,T), \mathbf{x}_t^k \sim q(\mathbf{x}_t^k | \mathbf{x}_0^k)} \log \sigma (-\beta  \omega(\gamma_t) ( \| \epsilon^w - \epsilon_\theta (\mathbf{x}_t^w, t) \|_2^2 \\ 
& \hspace{-2.5em} - \| \epsilon^w - \epsilon_{\text{ref}} (\mathbf{x}_t^w, t) \|_2^2  - ( \| \epsilon^l - \epsilon_\theta (\mathbf{x}_t^l, t) \|_2^2 - \| \epsilon^l - \epsilon_{\text{ref}} (\mathbf{x}_t^l, t) \|_2^2 ) ) ).
\end{aligned}
\end{equation}

\textbf{Eq.~\eqref{eq:7} shows that, at each denoising step, Diffusion-DPO optimizes $\theta$ to reduce the noise prediction error on winner samples, while increasing the corresponding error on loser samples.} Since the reverse diffusion update is parameterized by the predicted noise, the residual $\delta=\epsilon-\epsilon_\theta(x_t,t)$ directly determines the reverse denoising displacement at each timestep, thereby altering the reverse sampling trajectory itself. Consequently, reducing the residual on winner samples encourages trajectories to move toward regions assigned higher likelihood under the aligned model, while enlarging the residual on loser samples progressively redirects trajectories away from dispreferred regions. In this sense, Diffusion-DPO reshapes the transport of probability mass during the reverse diffusion process through trajectory-level modifications of the denoising dynamics. 

These prediction errors are measured using the standard Euclidean norm, treating all residual directions uniformly without considering the local geometric structure of the data distribution. The scalar $\beta$ controls the overall strength of this contrastive preference optimization.

\begin{figure}[t]
  \centering
  \includegraphics[width=0.9\linewidth]{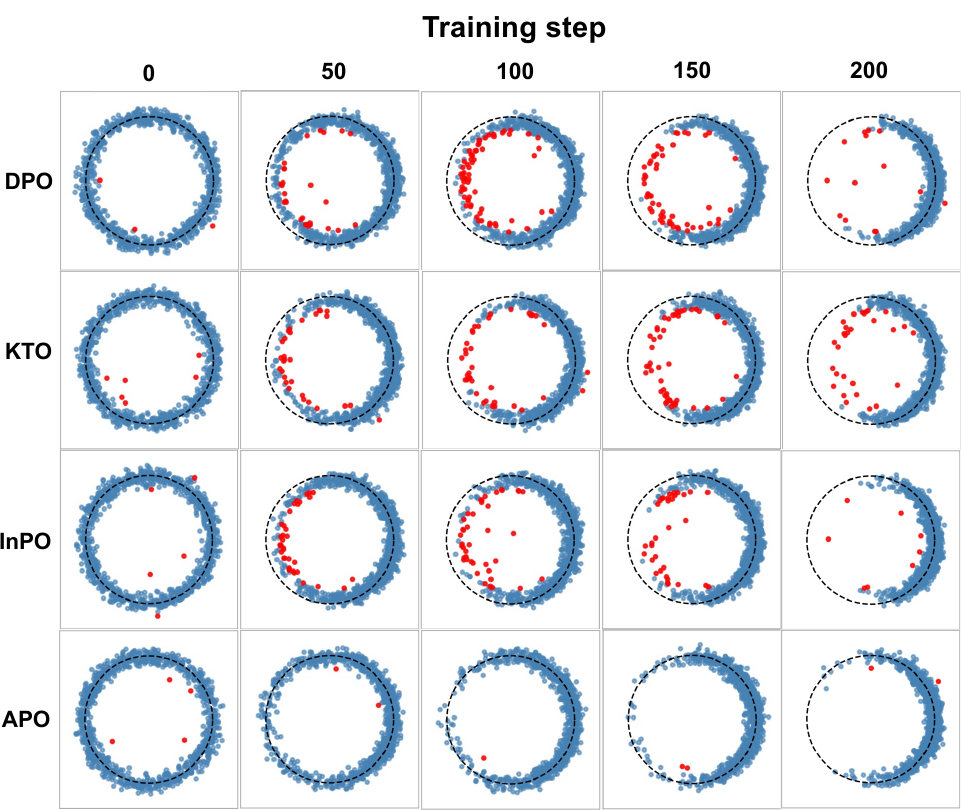}
  \caption{\textbf{Toy preference optimization on a manifold under different methods.} A DDPM with a 4-layer MLP noise predictor is pretrained on a unit circle embedded in $\mathbb{R}^{32}$. \textbf{Preference pairs are constructed by randomly sampling one point from each semicircle, with the right semicircle as the winner.} Columns denote training progress and rows denote methods.  \underline{Red} points indicate off-manifold samples. DPO, KTO, and InPO exhibit substantial off-manifold drift during alignment, whereas APO remains tightly concentrated around the underlying manifold while converging toward the winner region.}
  \label{fig:toy2d}
\end{figure}

\subsection{Off-Manifold Drift in Diffusion-DPO}
\label{sec:limitations}

As discussed in \Cref{sec:preference_optimization}, the training objective in Diffusion-DPO compares the prediction errors of the winner and the loser samples using a uniform Euclidean metric, treating all reverse-trajectory update directions in $\mathbb{R}^D$ equally. Under the manifold hypothesis \cite{bengio2013representation, fefferman2016testing}, however, natural data concentrate near a low-dimensional latent submanifold $\mathcal{M} \subset \mathbb{R}^D$, and the ambient space locally decomposes into tangent directions $T_x\mathcal{M}$ and normal directions $N_x\mathcal{M}$. These directions play fundamentally different geometric roles. Reverse-trajectory displacements along tangent directions tend to preserve semantic consistency and remain close to the data manifold, whereas displacements along normal directions are more likely to move trajectories away from the manifold and degrade image quality. DPO-style diffusion alignment methods measure all trajectory-update discrepancies isotropically under the same Euclidean geometry, without distinguishing geometry-preserving directions from geometry-sensitive ones. As a result, updates along geometry-sensitive manifold-normal directions are insufficiently constrained, while manifold-aligned tangent directions become unnecessarily restricted for semantic adjustment. We refer to this geometric failure mode as \textbf{off-manifold drift}.

We illustrate this phenomenon with a minimal example (\Cref{fig:toy2d}), where a DDPM pretrained on a unit circle embedded in $\mathbb{R}^{32}$ is fine-tuned. In this setup, the left semicircle corresponds to low-reward (loser) regions, while the right semicircle corresponds to high-reward (winner) regions. \textbf{For each preference pair, we randomly sample one point from each semicircle and designate the point from the right semicircle as the winner.} We compare DPO and its Euclidean-metric variants, all of which progressively shift probability mass toward the preferred region, yet exhibit noticeable off-manifold drift during training, producing large numbers of off-manifold samples (\textbf{red}). 

This observation motivates the central question of our work:
\textbf{can we design a geometry-aware anisotropic regularizer that adapts to the local structure of the diffusion representation space?} We provide a positive answer in the next section.

\section{Method}
\label{sec:method}
To address this issue, we introduce Anisotropic Geometry-Aware Preference Optimization (APO), which replaces the isotropic Euclidean penalty in Diffusion-DPO with a Jacobian-induced anisotropic geometry.

Our method consists of three components:
\begin{itemize}[leftmargin=1.5em, topsep=0pt, partopsep=0pt, parsep=0pt]
\item In \autoref{sec:preliminary}, we derive a local geometric characterization of the diffusion model via the Jacobian of the reference score field, which reveals an anisotropic decomposition into tangential and normal directions on the data manifold.
\item In \autoref{sec:apo}, we construct a Jacobian-induced quadratic form that defines a Riemannian metric on the noise prediction space, leading to an anisotropic preference optimization objective.
\item In \autoref{sec:imple}, we provide an efficient implementation for integrating our method into the diffusion alignment framework to streamline the computation.
\end{itemize}

\begin{figure}[t]
\centering
\includegraphics[width=\linewidth]{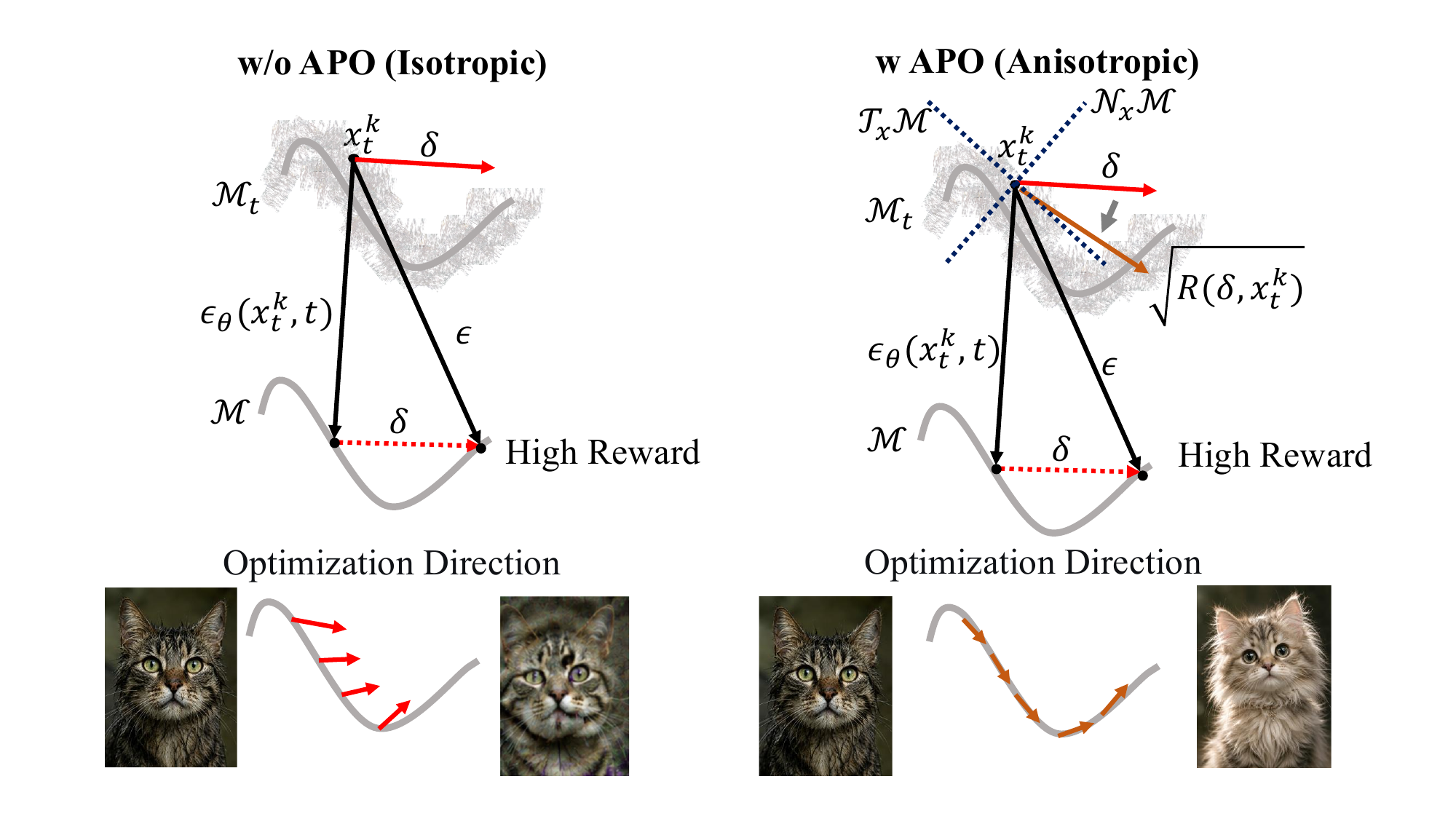}
\caption{\textbf{Schematic overview of our method}. Standard Diffusion-DPO measures the noise-prediction residual $\delta=\epsilon-\epsilon_\theta(x_t^k,t)$ under an isotropic Euclidean metric, treating tangent and normal directions equally and allowing updates to drift away from the local manifold $\mathcal{M}_t$. APO instead employs a Jacobian-induced anisotropic metric $R(\delta;x_t^k)$ from the reference model, which strengthens regularization along geometry-sensitive directions while allowing more flexible movement along manifold-aligned directions}
\label{fig:schematic}
\end{figure}

\subsection{Preliminaries}
\label{sec:preliminary}
\paragraph{Manifold hypothesis and diffusion} Under the manifold hypothesis, natural data concentrate near a low-dimensional latent submanifold $\mathcal{M} \subset \mathbb{R}^D$. While the forward diffusion process perturbs samples away from $\mathcal{M}$, the resulting distribution at time $t$, $p_t(x_t) = p_{\text{data}} * \mathcal{N}(0, \sigma_t^2 I)$ can be viewed as a Gaussian-smoothed version of the data distribution. Importantly, although $x_t$ no longer lies exactly on $\mathcal{M}$, the density $p_t$ retains its geometric structure: probability mass remains concentrated in a tubular neighborhood around the underlying manifold, with sharp variation in normal directions and relatively flat variation along tangent directions.

\paragraph{Score geometry and local anisotropy} The score function $\nabla_{x_t} \log p_t(x_t)$ characterizes how the density changes locally, and thus encodes geometric information of the distribution at scale $\sigma_t$. In particular, prior work has shown that its local variation exhibits anisotropy aligned with the underlying manifold structure~\cite{SM, pidstrigach2022score, stanczuk2024diffusion, bengio2013representation}: directions corresponding to rapid density change are associated with off-manifold (normal) perturbations, while directions with slow variation correspond to on-manifold (tangent) perturbations. In diffusion models, the noise predictor approximates the score function up to a scaling factor: $\epsilon_\theta(x_t,t) \approx -\sigma_t \nabla_{x_t} \log p_t(x_t)$. Therefore, the Jacobian of the noise predictor with respect to the input, $J_{\text{ref}}(x_t, t) = \frac{\partial \epsilon_{\text{ref}}(x_t, t)}{\partial x_t}$, describes how the predicted score field changes under infinitesimal perturbations of $x_t$. Its singular structure therefore reflects the local anisotropy of $p_t$: directions associated with large singular values correspond to perturbations that cause rapid variation of the score field, while directions associated with small singular values correspond to relatively insensitive directions.

\paragraph{Jacobian-Induced Sensitivity Weighting} In Diffusion-DPO, the optimization objective is defined through noise-prediction error $\delta = \epsilon - \epsilon_\theta(x_t,t)$, which lives in the output space of the noise predictor. However, the geometric structure encoded by the score field is characterized through the variation of $\epsilon_{\mathrm{ref}}(x_t,t)$ with respect to perturbations of the input variable $x_t$. To connect these two spaces, we use the Jacobian of the reference model to transport output-space perturbations back to the local coordinate system of the input. Consider a first-order perturbation of the reference score field:
\begin{equation}
\epsilon_{\mathrm{ref}}(x_t+\Delta x,t) \approx \epsilon_{\mathrm{ref}}(x_t,t) + J_{\mathrm{ref}}\,\Delta x.
\end{equation}
Under this local linearization, the Jacobian maps infinitesimal perturbations in the input space to corresponding changes in the output score field. Conversely, for a given output-space perturbation $\delta$, the vector
\begin{equation}
\delta_x := J_{\mathrm{ref}}^\top \delta
\label{eq:pullback}
\end{equation}
measures how strongly $\delta$ aligns with locally sensitive directions of the score field in the input space. Intuitively, prediction errors aligned with directions along which the score changes rapidly induce larger transported responses, while errors aligned with insensitive directions induce weaker responses. This Jacobian-transported directional variation will serve as the basis for our anisotropic regularization.

\subsection{APO: Anisotropic Geometry-Aware Preference Optimization}
\label{sec:apo}
We now formalize the geometric intuition developed in \autoref{sec:preliminary} into a concrete optimization objective. Let $J_{\mathrm{ref}}^k := \frac{\partial\epsilon_{\mathrm{ref}}(x_t^k,t)}{\partial x_t^k}$ be the Jacobian of the reference noise predictor at $x_t^k$, with SVD $J_{\mathrm{ref}}^k = U\Sigma V^\top$, where $\Sigma = \mathrm{diag}(\sigma_1,\dots,\sigma_D)$ and $U,V$ are orthonormal. As established in \autoref{sec:preliminary}, large singular values $\sigma_i$ correspond to directions along which the score field varies most rapidly. The right singular vectors $v_i$ define these directions in the input space, while the corresponding left singular vectors satisfy $u_i = J_{\mathrm{ref}}v_i/\sigma_i$, representing their images in the output (noise-prediction) space.

\begin{proposition}[APO Norm and Metric]
\label{prop:apo-metric}
For any noise-prediction error $\delta \in \mathbb{R}^D$, define the
anisotropic quadratic form
\begin{equation}
    R(\delta;\, x^k_t)
    \;\triangleq\;
    \|\delta\|^2 + \lambda\|J^{k\top}_{\mathrm{ref}}\,\delta\|^2
    \;=\;
    \sum_{i=1}^{D}(1+\lambda\sigma_i^2)\,(u_i^\top\delta)^2,
    \label{eq:R-def}
\end{equation}
where $\lambda>0$ is the anisotropic regularization coefficient.
Then $\|\delta\|_A := \sqrt{R(\delta;\,x^k_t)}$ defines a norm on $\mathbb{R}^D$,
and $d_A(\delta_1,\delta_2):=\|\delta_1-\delta_2\|_A$ defines a metric.
\end{proposition}

\begin{proof}
Write $R(\delta;\,x^k_t) = \delta^\top M \delta$, where
\[
    M \;:=\; I + \lambda\, J^k_{\mathrm{ref}} J^{k\top}_{\mathrm{ref}}
      \;=\; \sum_{i=1}^D (1+\lambda\sigma_i^2)\,u_i u_i^\top.
\]
Since $\lambda \ge 0$ and $J^k_{\mathrm{ref}}J^{k\top}_{\mathrm{ref}} \succeq 0$,
we have $M \succeq I \succ 0$, so $\langle\delta_1,\delta_2\rangle_M
:= \delta_1^\top M\delta_2$ is an inner product on $\mathbb{R}^D$ and
$\|\cdot\|_A = \sqrt{\langle\cdot,\cdot\rangle_M}$ is its induced norm.
The three axioms follow directly:

\textbf{(1) Positive definiteness:} $M\succ 0 \;\Rightarrow\;
\delta^\top M\delta > 0\ (\delta\neq\mathbf{0}),\quad
\mathbf{0}^\top M\mathbf{0} = 0.$

\textbf{(2) Homogeneity:}
$\|\alpha\delta\|_A = \sqrt{\alpha^2\delta^\top M\delta} = |\alpha|\|\delta\|_A.$

\textbf{(3) Triangle inequality:}
$\|\delta_1+\delta_2\|_A^2
= \|\delta_1\|_A^2 + 2\langle\delta_1,\delta_2\rangle_M + \|\delta_2\|_A^2
\le \|\delta_1\|_A^2 + 2\|\delta_1\|_A\|\delta_2\|_A + \|\delta_2\|_A^2
= (\|\delta_1\|_A+\|\delta_2\|_A)^2,$
where the inequality applies Cauchy--Schwarz to $\langle\cdot,\cdot\rangle_M$.
\end{proof}

\begin{remark}
\label{rmk}
At high noise levels, the singular values of $J^k_{\mathrm{ref}}$ become nearly uniform, so $(1+\lambda\sigma_i^2)\to c$ for all $i$ and $\|\cdot\|_A$ reduces to a scalar multiple of the Euclidean norm, recovering standard Diffusion-DPO. Components of $\delta$ aligned with $u_i$ directions of large $\sigma_i$ receive stronger penalties, encoding the off-manifold sensitivity of the reference score field. The identity term $\|\delta\|^2$ ensures that $R$ remains strictly positive definite regardless of the spectral structure of $J_{\mathrm{ref}}^k$. In particular, at low noise levels, possible rank-deficiency of $J_{\mathrm{ref}}^k$ in directions orthogonal to the score field is compensated by the identity term, preventing degenerate penalties.
\end{remark}

We replace the Euclidean norm used in ~\eqref{eq:7} with the anisotropic metric $R(\cdot;\,x_t^k)$, yielding:
\begin{equation}
  \ell_A^k
  \;:=\; R(\epsilon^k - \epsilon_\theta(x_t^k,t);\,x_t^k)
       - R(\epsilon^k - \epsilon_{\mathrm{ref}}(x_t^k,t);\,x_t^k).
  \label{eq:ell-A}
\end{equation}
Intuitively, $\ell_A^k$ measures the relative denoising quality of the current policy against the reference model under the geometry-aware anisotropic metric. The same Jacobian $J_{\mathrm{ref}}^k$ is used for both terms, ensuring a symmetric comparison.
Following the Jensen-bound and SNR-reweighting steps of Diffusion-DPO~\citep{DPO},
the final APO training objective is
\begin{equation}
  \mathcal{L}_{\mathrm{APO}}(\theta)
  = -\mathbb{E}_{\mathbf{x}_0^k \sim D,\, t \sim U(0,T),\, \mathbf{x}_t^k \sim q(\mathbf{x}_t^k|\mathbf{x}_0^k)}\log \sigma (-\beta  \omega(\gamma_t)\,\!(\ell_A^w - \ell_A^l)).
  \label{eq:adpo-loss}
\end{equation}
Setting $\lambda=0$ recovers Diffusion-DPO exactly.

\subsection{Efficient Implementation}
\label{sec:imple}
Evaluating the anisotropic metric $R(\delta; x_t^k)$ defined in ~\eqref{eq:R-def}, as instantiated in ~\eqref{eq:ell-A} ~\eqref{eq:adpo-loss}, requires the vector-Jacobian product $(J_{\mathrm{ref}}^{k})^\top\delta$, which is computed via reverse-mode autodiff. For large-scale architectures, this operation retains the full activation graph of $\epsilon_{\mathrm{ref}}$ and incurs a memory footprint of $O(|\theta_{\mathrm{ref}}|)$, limiting practical scalability. We therefore develop a forward-only estimator that replaces the reference backward pass with a single additional forward evaluation of $\epsilon_{\mathrm{ref}}$. By the Hutchinson trace identity, for any fixed vector $\delta \in \mathbb{R}^{D}$:
\begin{equation}
  \|(J_{\mathrm{ref}}^{k})^\top\delta\|^2
  =
  \mathbb{E}_{v\sim\mathcal{N}(0,I)} \,
  \!\Bigl[\bigl((J_{\mathrm{ref}}^{k}\,v)^\top\delta\bigr)^2\Bigr].
  \label{eq:hutchinson}
\end{equation}

For a perturbation radius $r > 0$ and $n_v$ independent random probe directions $v^{(j)}\sim\mathcal{N}(0,I)$, $j=1,\ldots,n_v$, we approximate the JVP for each probe by a first-order finite difference:
\begin{equation}
  \widehat{J_{\mathrm{ref}}^{k}\,v^{(j)}}
  :=
  \frac{\epsilon_{\mathrm{ref}}(x_t^{k} + r v^{(j)},\,t)
  - \epsilon_{\mathrm{ref}}(x_t^{k},\,t)}{r},
  \label{eq:fd-jvp}
\end{equation}
which yields the $n_v$-sample estimator
\begin{equation}
\begin{aligned}
  \widehat{\|(J_{\mathrm{ref}}^{k})^\top\delta\|^2} := \frac{1}{n_v}\sum_{j=1}^{n_v} \bigl(\widehat{J_{\mathrm{ref}}^{k}\,v^{(j)}}^{\!\top}\delta\bigr)^2 
  =
  \frac{1}{n_v}\sum_{j=1}^{n_v}
  \frac{1}{r^2}
  \Bigl(
  \bigl[\epsilon_{\mathrm{ref}}(x_t^{k}+r v^{(j)},\,t)
  -\epsilon_{\mathrm{ref}}(x_t^{k},\,t)\bigr]^{\!\top}\!\delta
  \Bigr)^2.
  \label{eq:fd-estimator}
\end{aligned}
\end{equation}

A Taylor expansion of $\epsilon_{\mathrm{ref}}$ around $x_t^{k}$ yields:
\begin{equation}
  \epsilon_{\mathrm{ref}}(x_t^{k}+r v) = \epsilon_{\mathrm{ref}}(x_t^{k}) + r J_{\mathrm{ref}}^k v + \frac{r^2}{2} H_{\epsilon_{\mathrm{ref}}}[v,v] + O(r^3),
\end{equation}
where $H_{\epsilon_{\mathrm{ref}}}$ denotes the Hessian of $\epsilon_{\mathrm{ref}}$.
Substituting into the finite-difference estimator, the dominant bias term arises from the inner product of the quadratic Hessian term with itself.
The cross-term between the linear term $J_{\mathrm{ref}}^k v$ and the quadratic Hessian term $H_{\epsilon_{\mathrm{ref}}}[v,v]$ is a degree-three polynomial in $v$,
whose expectation vanishes identically under symmetric Gaussian measure.
The remaining second-order bias satisfies:
\begin{equation}
  \mathbb{E}_{v}\!\Bigl[
  \widehat{\|(J_{\mathrm{ref}}^{k})^\top\delta\|^2}
  \Bigr]
  =
  \|(J_{\mathrm{ref}}^{k})^\top\delta\|^2
  +
  O\left(r^2 \cdot \|H_{\epsilon_{\mathrm{ref}}}\|_F^2\right).
  \label{eq:bias}
\end{equation}
In practice, setting $r \ll 1$ renders this $O(r^2)$ residual negligible relative to the signal magnitude
$\|(J_{\mathrm{ref}}^{k})^\top\delta\|^2$, while keeping the
finite-difference step well within the linear regime of $\epsilon_{\mathrm{ref}}$.

Replacing $R(\delta; x_t^{k})$ in \Cref{eq:ell-A} with $\hat{R}(\delta; x_t^{k})$,
defined by substituting \Cref{eq:fd-jvp} into \Cref{eq:R-def}, gives the forward-only
anisotropic metric:
\begin{equation}
  \hat{R}(\delta; x_t^{k})
  :=
  \|\delta\|^2
  +
  \frac{\lambda}{n_v}\sum_{j=1}^{n_v}
  \frac{1}{r^2}
  \Bigl(
  \bigl[\epsilon_{\mathrm{ref}}(x_t^{k}+r v^{(j)},\,t)
  -\epsilon_{\mathrm{ref}}(x_t^{k},\,t)\bigr]^{\!\top}\!\delta
  \Bigr)^2.
  \label{eq:fd-metric}
\end{equation}
Substituting $\hat{R}$ for $R$ in \Cref{eq:adpo-loss} defines the APO training objective. For each preference pair $(x_0^w, x_0^l)$ and sampled timestep $t$, the additional cost of evaluating $\hat{R}$ reduces to $n_v + 1$ forward passes of $\epsilon_{\mathrm{ref}}$ under \texttt{torch.no\_grad()}: a baseline pass at $x_t^{k}$ (shared across all probe directions) and $n_v$ perturbed passes at $x_t^{k} + r v^{(j)}$, $j = 1, \ldots, n_v$. The baseline pass is already required by standard Diffusion-DPO and thus incurs no additional cost. Crucially, for a given sample $k$, the same perturbed outputs serve both $\hat{R}(\delta_{\theta}^{k}; x_t^{k})$ and $\hat{R}(\delta_{\mathrm{ref}}^{k}; x_t^{k})$, since only the inner product partner $\delta$ changes between the two evaluations. This eliminates the reference backward pass entirely, reducing the activation memory from $O(|\theta_{\mathrm{ref}}|)$ to the negligible cost of storing $r v$. 
\begin{algorithm}[t]
\caption{APO: Single Training Step (Hutchinson-FD Variant)}
\label{alg:apo-fd}
\begin{algorithmic}[1]
\REQUIRE{%
  Pre-trained reference model $\epsilon_{\mathrm{ref}}$ (frozen);
  current policy $\epsilon_\theta$;
  preference pair $(\mathbf{x}_0^w,\mathbf{x}_0^l)$;
  hyperparameters $\beta$, $\lambda$, $r$, $n_v$;
  SNR weighting $\omega(\gamma_t)$.
}
\ENSURE{Updated parameters $\theta \leftarrow \theta - \eta \nabla_\theta \mathcal{L}_{\mathrm{APO}}$.}

\STATE Sample $t \sim \mathcal{U}(0,T)$ and $\epsilon \sim \mathcal{N}(0,I)$
       \COMMENT{Shared noise $\epsilon$}

\FOR{$k \in \{w,l\}$}
  \STATE $\delta_{\theta}^k \leftarrow \epsilon-\epsilon_\theta(\mathbf{x}_t^k,t)$;\enskip
         $\delta_{\mathrm{ref}}^k \leftarrow \epsilon-\epsilon_{\mathrm{ref}}(\mathbf{x}_t^k,t)$
         \COMMENT{Unified $\epsilon$}

  \STATE $p^k \leftarrow 0$;\enskip $q^k \leftarrow 0$
  
  \FOR{$j=1,\ldots,n_v$}
    \STATE Sample $v^{(j)}\sim\mathcal{N}(0,I)$

    \STATE $\Delta\epsilon^{k,j}
           \leftarrow
           \epsilon_{\mathrm{ref}}(\mathbf{x}_t^k+r\,v^{(j)},t)
           -\epsilon_{\mathrm{ref}}(\mathbf{x}_t^k,t)$

    \STATE $p^k \mathrel{+}=
           \frac{((\Delta\epsilon^{k,j})^\top\delta_\theta^k)^2}{n_vr^2}$;\enskip
           $q^k \mathrel{+}=
           \frac{((\Delta\epsilon^{k,j})^\top\delta_{\mathrm{ref}}^k)^2}{n_vr^2}$
  \ENDFOR

  \STATE $\ell_A^k \leftarrow
         \|\delta_\theta^k\|^2-\|\delta_{\mathrm{ref}}^k\|^2
         +\lambda(p^k-q^k)$
         \COMMENT{Anisotropic preference loss}
\ENDFOR

\STATE $\mathcal{L}_{\mathrm{APO}}
       \leftarrow
       -\log\sigma\!\left(
       -\beta\,\omega(\gamma_t)(\ell_A^w-\ell_A^l)\right)$

\STATE Backpropagate through $\mathcal{L}_{\mathrm{APO}}$ and update $\theta$
\end{algorithmic}
\end{algorithm}
\subsection{Analysis and Validation}
\label{sec:prop-comp}
\paragraph{Sanity check on the toy setting}
Returning to the toy experiment of \Cref{sec:limitations}, we apply APO under identical conditions. As shown in the bottom row of \Cref{fig:toy2d}, APO avoids the off-manifold drift and the runaway shift observed in standard DPO, maintaining manifold integrity while converging stably within the winner region, and requiring fewer alignment steps than standard DPO. This suggests that the Jacobian-induced metric effectively suppresses off-manifold drift even in this minimal setting, providing a clean qualitative validation before scaling to large text-to-image models.

\begin{figure}[t]
\centering
\includegraphics[width=\linewidth]{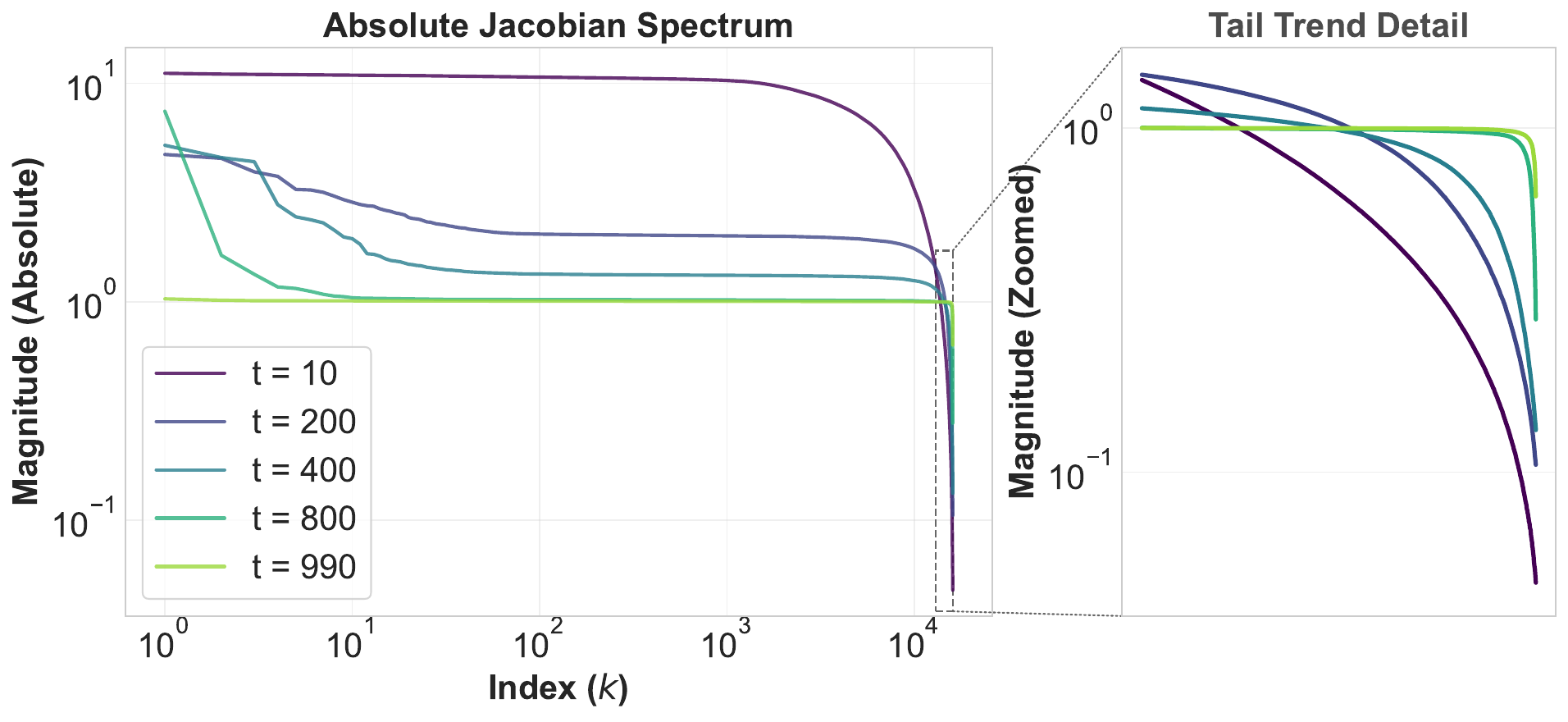}
\caption{\textbf{Spectrum analysis of the reference Jacobian on the Pick-a-Pic v2 dataset.} The plot shows the singular value distribution of $J_{\mathrm{ref}}$, illustrating the anisotropic structure captured by our approach.}
\label{fig:spectrum}
\end{figure}

\paragraph{Jacobian spectral analysis}
We further analyze the singular value spectrum of $J_{\mathrm{ref}}$ across diffusion timesteps on the Pick-a-Pic v2 dataset \cite{kirstain2023pick}, to empirically validate the induced geometric structure. The experimental results are shown in \Cref{fig:spectrum}: singular values are concentrated in a relatively stable high-energy band (approximately $1$ to $10$ in magnitude), while a sharp decay emerges beyond a critical index around $1.2\times10^4$ in a $16{,}384$-dimensional latent space. This effective rank decreases as the noise level increases. At larger diffusion timesteps (higher noise levels), the Jacobian spectrum becomes flatter and more uniform, consistent with the progressively isotropic geometry discussed in \Cref{rmk}.

\paragraph{Single-probe stability}
Although the Hutchinson estimator in ~\Cref{eq:fd-metric} is a stochastic approximation, a single probe direction ($n_v=1$) is already sufficient in practice. The resulting optimization remains highly stable, and this behavior arises from the temporal coherence of diffusion trajectories themselves. APO does not require recovering the full Jacobian structure at each timestep. The objective only depends on the dominant anisotropic direction induced by $J_{\mathrm{ref}}^\top \delta_t$. Along the diffusion trajectory, nearby timesteps share highly correlated geometric structure and similar local manifold orientation. As a result, although a single probe may introduce instantaneous stochasticity, its projected direction repeatedly aligns with the same underlying anisotropic component across neighboring noise scales. 

Consequently, the stochastic errors introduced by individual probes tend to fluctuate inconsistently across timesteps, while the persistent geometric component accumulates coherently throughout optimization. This creates a natural temporal averaging effect along the diffusion trajectory: unstable probe noise is progressively canceled, whereas stable manifold-aware directions are repeatedly reinforced. In practice, this reduces the variance to a low level, making $n_v=1$ already sufficient to produce stable optimization dynamics. Naturally, increasing the number of probe directions can further improve the accuracy and stability of the anisotropic estimate. However, the empirical gains quickly exhibit diminishing returns relative to the additional computational cost. Therefore, we retain the multi-probe formulation in ~\Cref{alg:apo-fd} for completeness, and provide a detailed ablation study on $n_v$ in ~\Cref{sec:nv}.

\section{Experiments}
Our experiments seek to answer three questions:
\begin{itemize}[leftmargin=1.5em, topsep=0pt, partopsep=0pt, parsep=0pt]
\item Does APO improve preference alignment quality and generation behavior across architectures and benchmarks?
\item Are the gains truly attributable to the proposed Jacobian-induced anisotropic geometry?
\item Under what regimes is the geometric regularization most beneficial?
\end{itemize}
\autoref{sec:main} presents the main quantitative and qualitative results. \autoref{sec:ablation} provides ablations and detailed analyses of the underlying geometric mechanism.

\subsection{Setup}
\paragraph{Models, Datasets, and Baselines} We conduct experiments on two widely used text-to-image backbones, Stable Diffusion 1.5 (SD1.5) and Stable Diffusion XL (SDXL)~\cite{sdxl}. All models are fine-tuned on the Pick-a-Pic v2 dataset~\cite{kirstain2023pick}, which provides 851,293 training preference pairs over 58,960 distinct prompts after removing ties. We compare our method against the unmodified base models (SD1.5 and SDXL), supervised fine-tuning (SFT) on winner images~\cite{SFT}, Diffusion-DPO~\cite{DPO}, InPO~\cite{InPO}, and Diffusion-SDPO~\cite{SDPO}. Additionally, on SD1.5, we also compare against Diffusion-KTO~\cite{KTO}, and on SDXL, we additionally compare against LPO~\cite{LPO}. All the above baselines employ the uniform Euclidean treatment of prediction errors. Our experiments are evaluated on the HPSv2~\cite{wu2023human}, Parti-Prompts~\cite{yu2022parti} and Pick-a-Picv2~\cite{kirstain2023pick} test sets, following standard protocols.
\paragraph{Evaluation} We assess alignment quality using three reward models that capture human preference: PickScore~\cite{kirstain2023pick}, HPS~\cite{wu2023human}, and the LAION Aesthetic Predictor~\cite{schuhmann2022laion}. We further report CLIP score~\cite{radford2021clip} for text-image alignment. Furthermore, we conducted a human evaluation on SDXL across four dimensions. We recruited 40 professional participants, each evaluating 50 different sets of images. Each set included the base model, DPO, and APO, with the order randomly shuffled. For each set, participants rated images across four dimensions: text-content alignment, image realism, aesthetic quality, and overall preference. We also evaluated the training efficiency on HPSv2 for SD1.5 and SDXL: we report the single-GPU (NVIDIA A100) training hours and peak GPU memory (GB) on SD1.5 and SDXL, along with a comparison against the standard DPO method.

\paragraph{Implementation Details} We train APO on all models for 900 optimization steps with a global batch size of 2048. For SD1.5, we perform full fine-tuning. For SDXL, we adopt LoRA-based fine-tuning to reduce memory and computational cost, using rank $r=32$, scaling factor $\alpha=32$, and dropout rate $0.05$. The text encoder remains frozen throughout training. Optimization is performed using AdamW with a cosine learning rate schedule. The learning rate follows the scaling rule $\frac{2000}{\beta} \times 2.048 \times 10^{-8}$. We use mixed-precision training (bf16) for all models. For APO-specific hyperparameters, we use $\lambda=0.8$ and $\beta=500$ for SD1.5, and $\lambda=1.0$ and $\beta=800$ for SDXL. The finite-difference perturbation radius is set to $r=10^{-3}$, and we use a single Hutchinson probe direction ($n_v=1$). During inference, we use 50 denoising steps with classifier-free guidance scales of 7.5 for SD1.5 and 5.0 for SDXL. For methods with publicly released checkpoints, we directly load the official weights~\cite{SDPO, KTO, LPO}.

\begin{table}[!ht]
\centering

\caption{\textbf{Main results and Win Rate comparison on SD1.5.} Left: Absolute scores on four benchmarks. Right: Win Rate (\%) of \textbf{APO (Ours)} vs. each baseline. The best results are \textbf{bolded}, second best are \underline{underlined}. APO achieves the best or second-best performance on nearly all metrics across all three evaluation sets, and the win-rate results further corroborate these gains.}

\label{tab:combined-results}

\footnotesize
\setlength{\tabcolsep}{2pt}

\begin{tabular}{@{\hspace{2pt}}
l l
cccc
p{1pt}
ccccc
@{\hspace{2pt}}
}

\toprule
\specialrule{0.05em}{0pt}{2pt}

& &
\multicolumn{4}{c}{Absolute Scores $\uparrow$}
&&
\multicolumn{5}{c}{Win Rate (\%) vs. Baseline $\uparrow$}
\\

\cmidrule{3-6}
\cmidrule{8-12}

Dataset
& Method
& Aes.
& Pick.
& CLIP
& IR.
&&
Aes.
& Pick.
& CLIP
& IR.
& Avg.
\\

\midrule

\multirow{7}{*}{HPSv2}

& Base (SD 1.5)
& 5.2473
& 21.037
& 33.258
& 0.1726
&&
\textbf{85.24}
& \textbf{87.45}
& \textbf{57.83}
& \textbf{85.71}
& \textbf{79.06}
\\

& SFT
& 5.7971
& 21.738
& 33.569
& 0.5162
&&
\textbf{53.61}
& \textbf{68.98}
& \textbf{57.06}
& \textbf{63.75}
& \textbf{60.85}
\\

& Diffusion-DPO
& 5.4192
& 21.551
& \underline{35.412}
& 0.3430
&&
\textbf{65.42}
& \textbf{70.13}
& \textbf{51.20}
& \textbf{66.99}
& \textbf{63.44}
\\

& SDPO
& 5.4365
& 21.715
& 34.532
& 0.3936
&&
\textbf{58.27}
& \textbf{63.83}
& \textbf{56.96}
& \textbf{61.45}
& \textbf{60.13}
\\

& Diffusion-KTO
& 5.6084
& 21.672
& 34.107
& \underline{0.6295}
&&
\textbf{60.08}
& \textbf{71.62}
& \textbf{54.18}
& \textbf{59.30}
& \textbf{61.30}
\\

& InPO
& \underline{5.7988}
& \underline{21.742}
& 34.680
& 0.5379
&&
47.70
& \textbf{64.15}
& \textbf{56.85}
& \textbf{58.38}
& \textbf{56.77}
\\

& \textbf{APO (Ours)}
& \textbf{5.8076}
& \textbf{22.014}
& \textbf{35.873}
& \textbf{0.6887}
&&
\textemdash
& \textemdash
& \textemdash
& \textemdash
& \textemdash
\\

\midrule

\multirow{7}{*}{Parti-Prompts}

& Base (SD 1.5)
& 5.1386
& 21.048
& 31.276
& 0.1543
&&
\textbf{83.69}
& \textbf{82.14}
& \textbf{54.92}
& \textbf{86.34}
& \textbf{76.77}
\\

& SFT
& 5.6872
& 21.725
& 31.697
& 0.4360
&&
\textbf{56.00}
& \textbf{64.25}
& \textbf{55.64}
& \textbf{67.13}
& \textbf{60.76}
\\

& Diffusion-DPO
& 5.3952
& 21.693
& 32.016
& 0.3987
&&
\textbf{69.78}
& \textbf{63.19}
& \textbf{54.66}
& \textbf{68.43}
& \textbf{64.02}
\\

& SDPO
& 5.4127
& 21.718
& 31.945
& 0.5215
&&
\textbf{55.31}
& \textbf{60.82}
& \textbf{56.14}
& \textbf{65.87}
& \textbf{59.54}
\\

& Diffusion-KTO
& 5.5601
& 21.711
& \underline{32.724}
& \underline{0.5392}
&&
\textbf{64.44}
& \textbf{72.96}
& 49.30
& \textbf{65.01}
& \textbf{62.93}
\\

& InPO
& \textbf{5.7691}
& \underline{21.736}
& 32.681
& 0.5381
&&
42.97
& \textbf{59.26}
& \textbf{53.20}
& \textbf{58.38}
& \textbf{53.45}
\\

& \textbf{APO (Ours)}
& \underline{5.7310}
& \textbf{22.114}
& \textbf{32.768}
& \textbf{0.5417}
&&
\textemdash
& \textemdash
& \textemdash
& \textemdash
& \textemdash
\\

\midrule

\multirow{7}{*}{Pick-a-Pic}

& Base (SD 1.5)
& 5.3279
& 21.008
& 32.086
& 0.2492
&&
\textbf{84.02}
& \textbf{85.55}
& \textbf{53.84}
& \textbf{83.09}
& \textbf{76.63}
\\

& SFT
& 5.5504
& 21.691
& 32.351
& 0.4630
&&
\textbf{54.71}
& \textbf{67.91}
& \textbf{55.36}
& \textbf{61.45}
& \textbf{59.86}
\\

& Diffusion-DPO
& 5.3862
& 21.583
& 33.637
& 0.3842
&&
\textbf{66.85}
& \textbf{63.23}
& \textbf{54.99}
& \textbf{67.35}
& \textbf{63.11}
\\

& SDPO
& \underline{5.6870}
& 21.730
& \underline{34.493}
& 0.5897
&&
\textbf{51.31}
& \textbf{63.54}
& \textbf{50.37}
& \textbf{60.50}
& \textbf{56.43}
\\

& Diffusion-KTO
& 5.6645
& 21.252
& 33.802
& 0.4992
&&
\textbf{64.44}
& \textbf{72.96}
& \textbf{54.30}
& \textbf{65.01}
& \textbf{64.18}
\\

& InPO
& 5.5471
& \underline{21.754}
& 34.468
& \underline{0.6231}
&&
\textbf{52.30}
& \textbf{60.67}
& \textbf{53.82}
& \textbf{52.53}
& \textbf{54.83}
\\


& \textbf{APO (Ours)}
& \textbf{5.6892}
& \textbf{21.956}
& \textbf{34.784}
& \textbf{0.6698}
&&
\textemdash
& \textemdash
& \textemdash
& \textemdash
& \textemdash
\\

\bottomrule

\end{tabular}
\end{table}
\subsection{Main Results}
\label{sec:main}
\paragraph{Quantitative Comparison with Prior Methods}
\Cref{tab:combined-results} and \Cref{tab:combined-results-sdxl} summarize the absolute scores and win rates of APO against all baselines. APO consistently achieves near-state-of-the-art performance on both SD1.5 and SDXL. On SD1.5, APO attains the best PickScore across all three evaluation sets; for example, on HPSv2, APO improves PickScore from 21.551 (Diffusion-DPO) to 22.014, while substantially increasing ImageReward from 0.3430 to 0.6887. The win-rate results further corroborate these gains: APO achieves an average win rate of 79.06\% against the unmodified SD1.5 model on HPSv2, and consistently outperforms reward-based baselines such as Diffusion-KTO, SDPO, and InPO across diverse benchmarks. These improvements remain stable across different prompt distributions (HPSv2, Parti-Prompts, and Pick-a-Pic), suggesting that the Jacobian-induced Riemannian metric generalizes beyond a specific prompt domain. Moreover, human evaluation results on SDXL show that APO consistently improves alignment with human preferences across all four criteria, achieving particularly strong gains in realism and aesthetic quality, where APO is preferred in \textbf{78\%} and \textbf{76\%} of comparisons, respectively.

\begin{figure}[t]
\centering
\includegraphics[width=\linewidth]{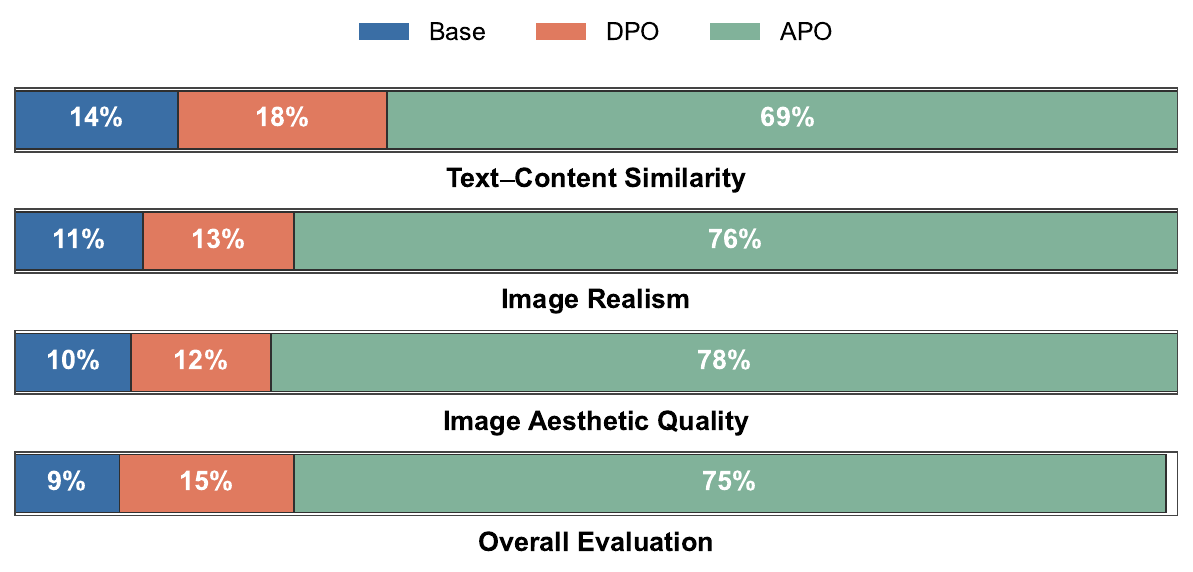}
\caption{\textbf{Human evaluation results.} We compare Base-SDXL, DPO-SDXL, and APO-SDXL across four dimensions: Text–to-Content Similarity, Image Realism, Aesthetic Quality, and Overall Preference. APO-SDXL outperforms both baselines in all categories.}
\label{fig:human_eval}
\end{figure}

\begin{figure}[t]
\centering
\includegraphics[width=\linewidth]{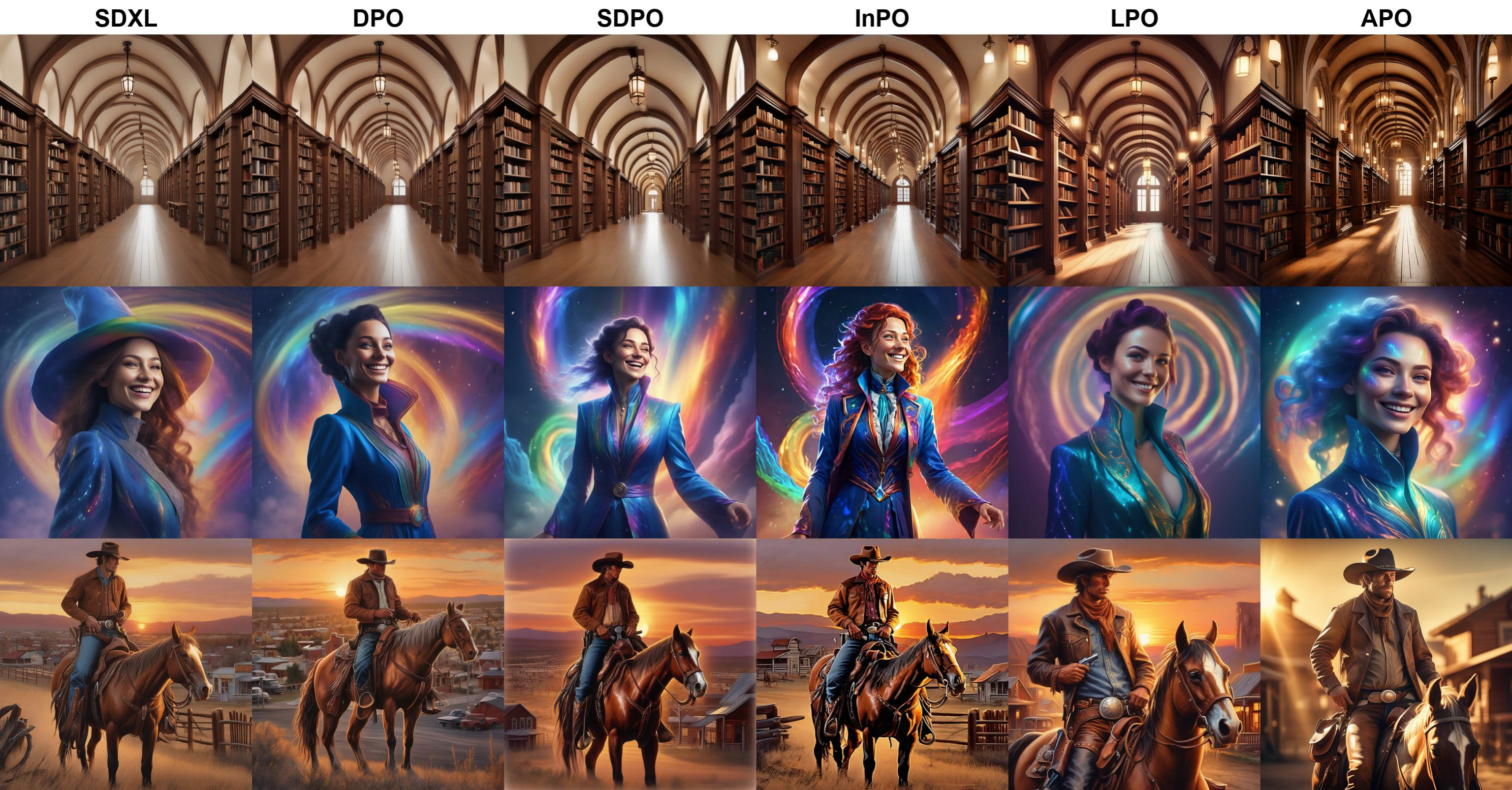}
\caption{\textbf{Qualitative comparison of various preference optimization methods based on SDXL.} Prompt: 1) \textit{A long library corridor with arched ceiling, symmetrical wooden bookshelves, warm pendant lighting, photorealistic.} 2) \textit{A smiling beautiful sorceress wearing a high necked blue suit surrounded by swirling rainbow aurora, hyper-realistic, cinematic, post-production.} 3) \textit{Cowboy on horseback, sunset small town background, medium close-up, detailed leather chaps and revolver holster.} }
\label{fig:example}
\end{figure}

\begin{table}[t]
\centering
\caption{\textbf{Training Efficiency on HPSv2.} We report Total Training Hours (h) and peak GPU memory (GB) on SD1.5 and SDXL. APO reduces the total training time to roughly half of standard Diffusion-DPO on both models, primarily due to faster convergence and fewer optimization steps. Despite introducing additional forward evaluations for Jacobian estimation, APO incurs only a modest increase in peak GPU memory of around 5\%. Detailed analysis is provided in \Cref{sec:main}.}
\label{tab:efficiency-results}
\footnotesize
\setlength{\tabcolsep}{4pt}
\begin{tabular}{lccccc}
\toprule
& \multicolumn{2}{c}{SD1.5} & & \multicolumn{2}{c}{SDXL} \\
\cmidrule{2-3} \cmidrule{5-6}
Method & Time Hours(h) $\downarrow$ & Peak Mem. (GB) $\downarrow$ & & Time Hours (h) $\downarrow$ & Peak Mem. (GB) $\downarrow$ \\
\midrule
Diffusion-DPO  & 205.6  & 14.6 & & 198.4  & 15.8 \\
\textbf{APO (Ours)}     & 112.7  & 15.4 & & 105.5  & 16.5 \\
\bottomrule
\end{tabular}
\end{table}

\begin{figure}[t]
\centering
\includegraphics[width=0.32\linewidth]{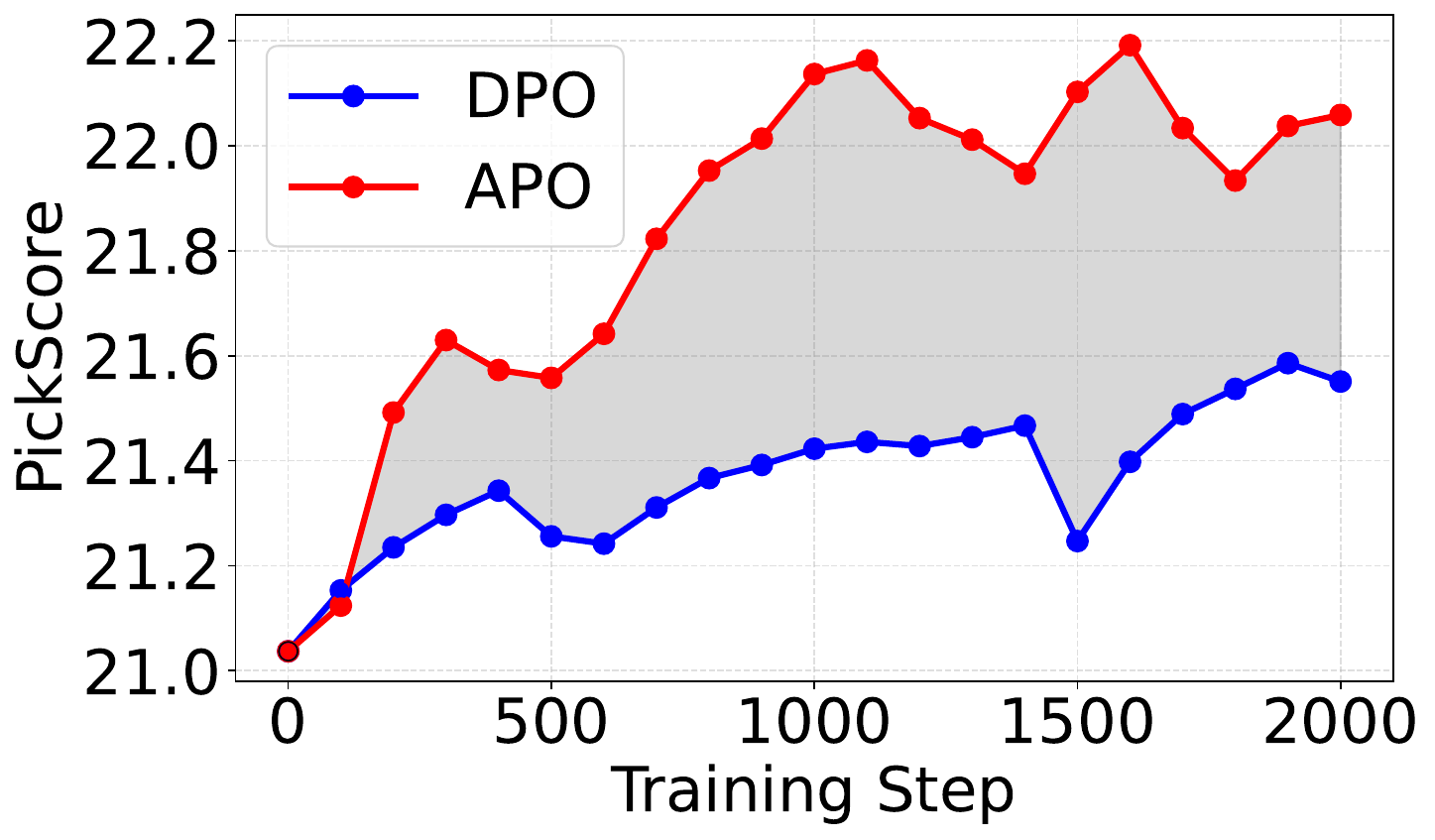}
\hfill
\includegraphics[width=0.32\linewidth]{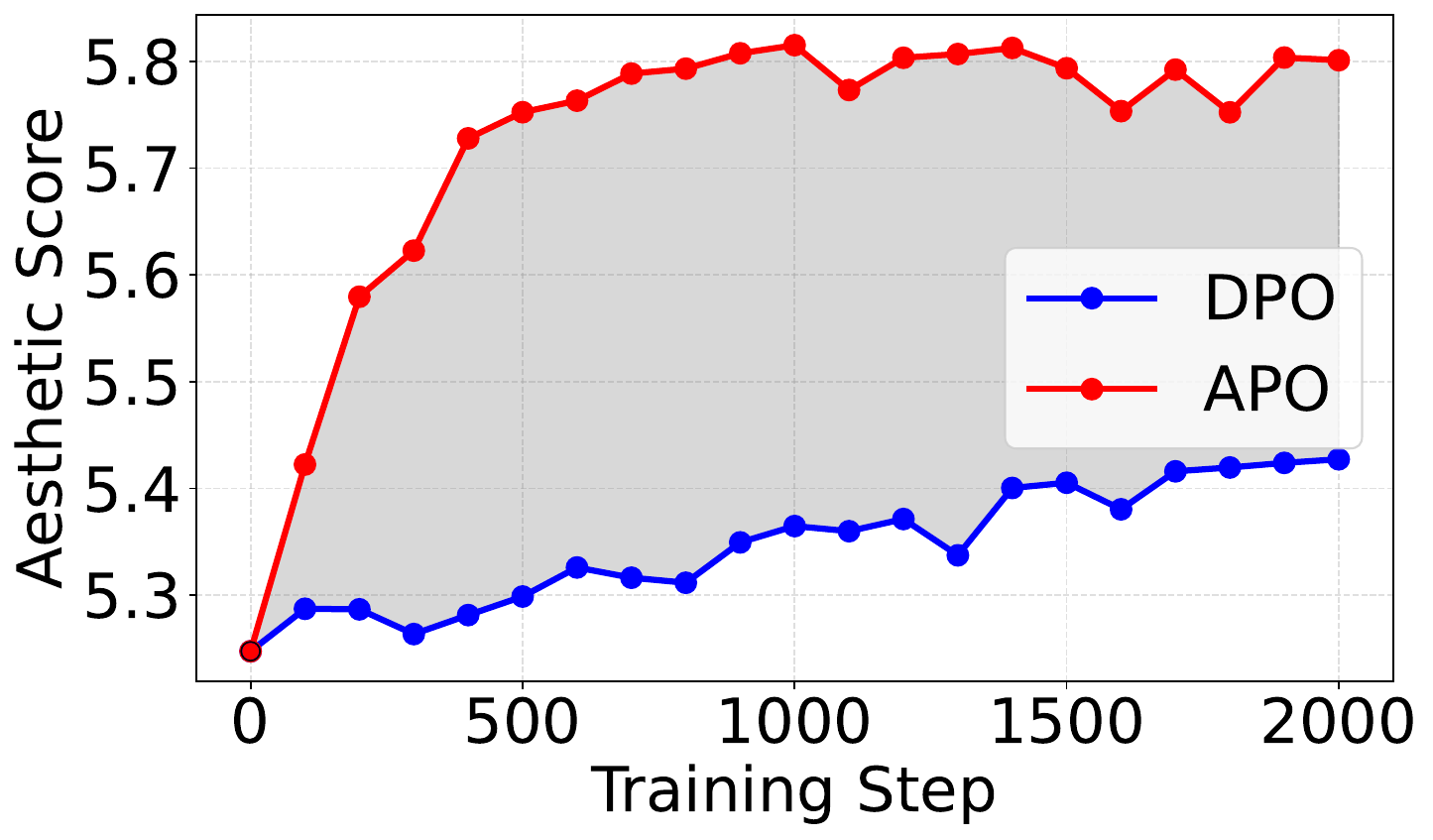}
\hfill
\includegraphics[width=0.32\linewidth]{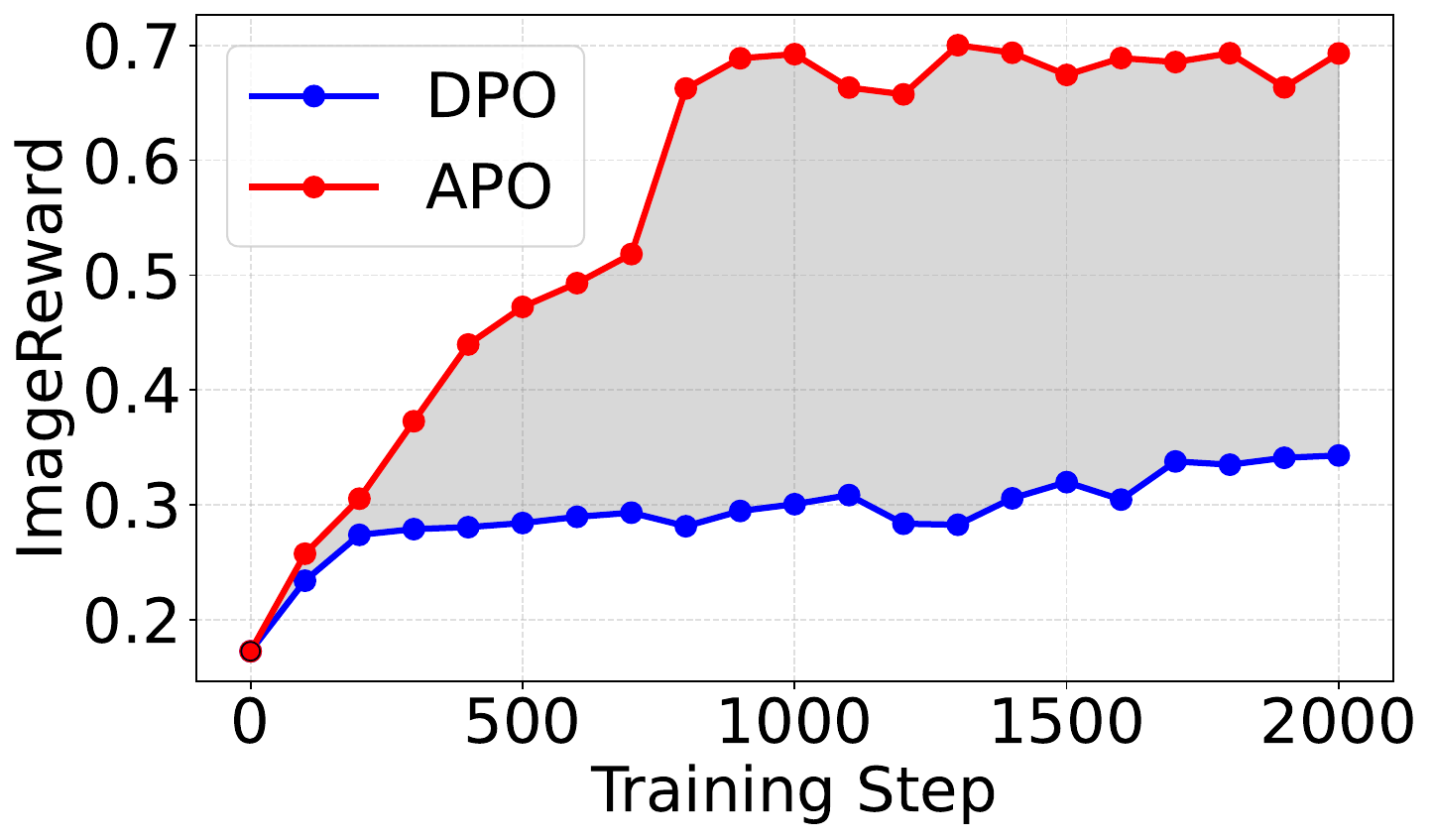}
\caption{\textbf{Training Dynamics of Reward Scores.} We plot the (a) PickScore, (b) Aesthetic Score, and (c) ImageReward of SD1.5 models throughout the training process. APO achieves competitive performance within significantly fewer training steps. }
\label{fig:dynamics}
\end{figure}

\begin{figure}[t]
\centering
\begin{subfigure}[b]{0.48\linewidth}
    \centering
    \includegraphics[width=\linewidth]{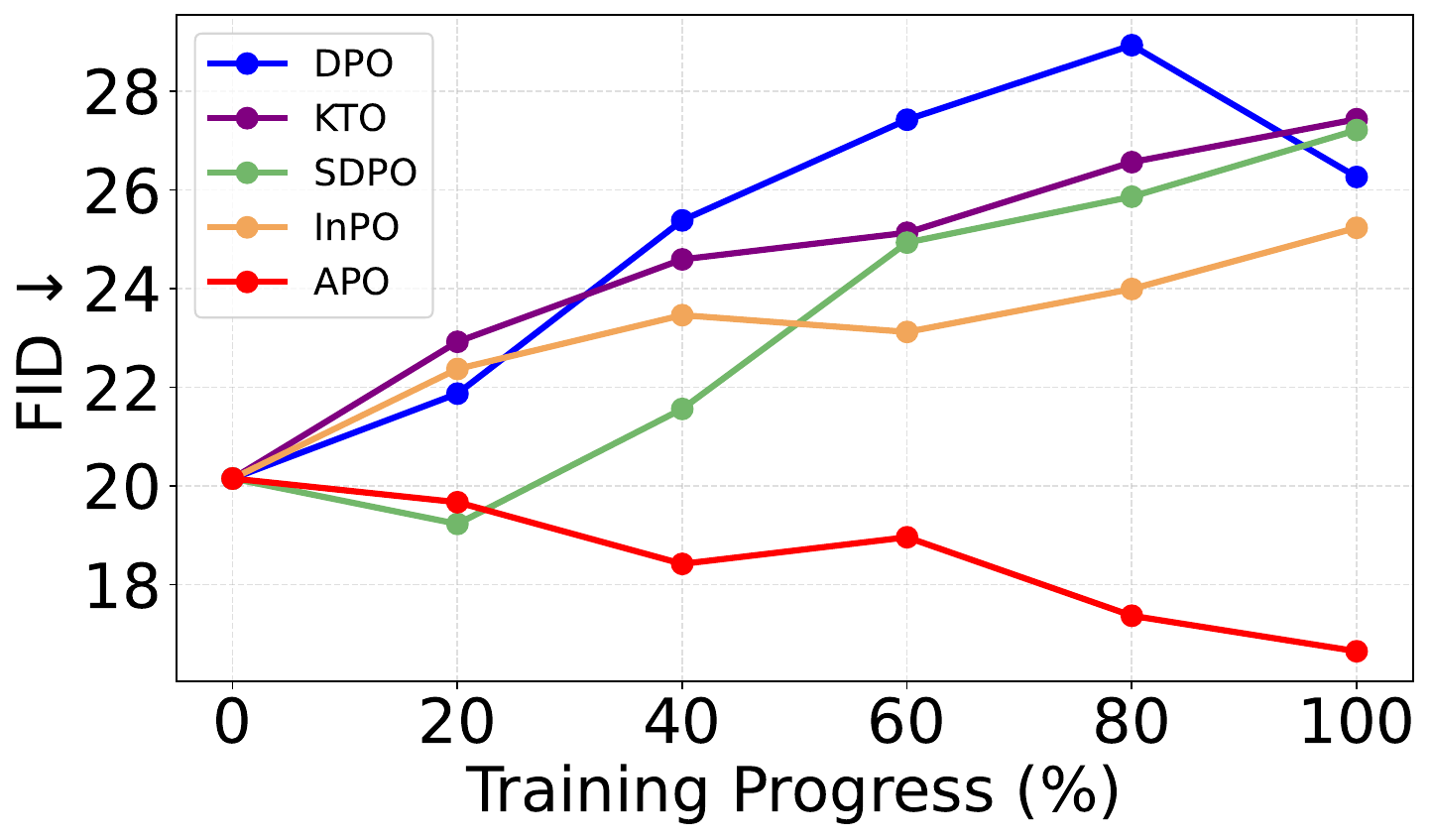}
    \caption{Distributional Stability During Alignment on SD1.5.}
    \label{fig:fid_a}
\end{subfigure}
\hfill
\begin{subfigure}[b]{0.48\linewidth}
    \centering
    \includegraphics[width=\linewidth]{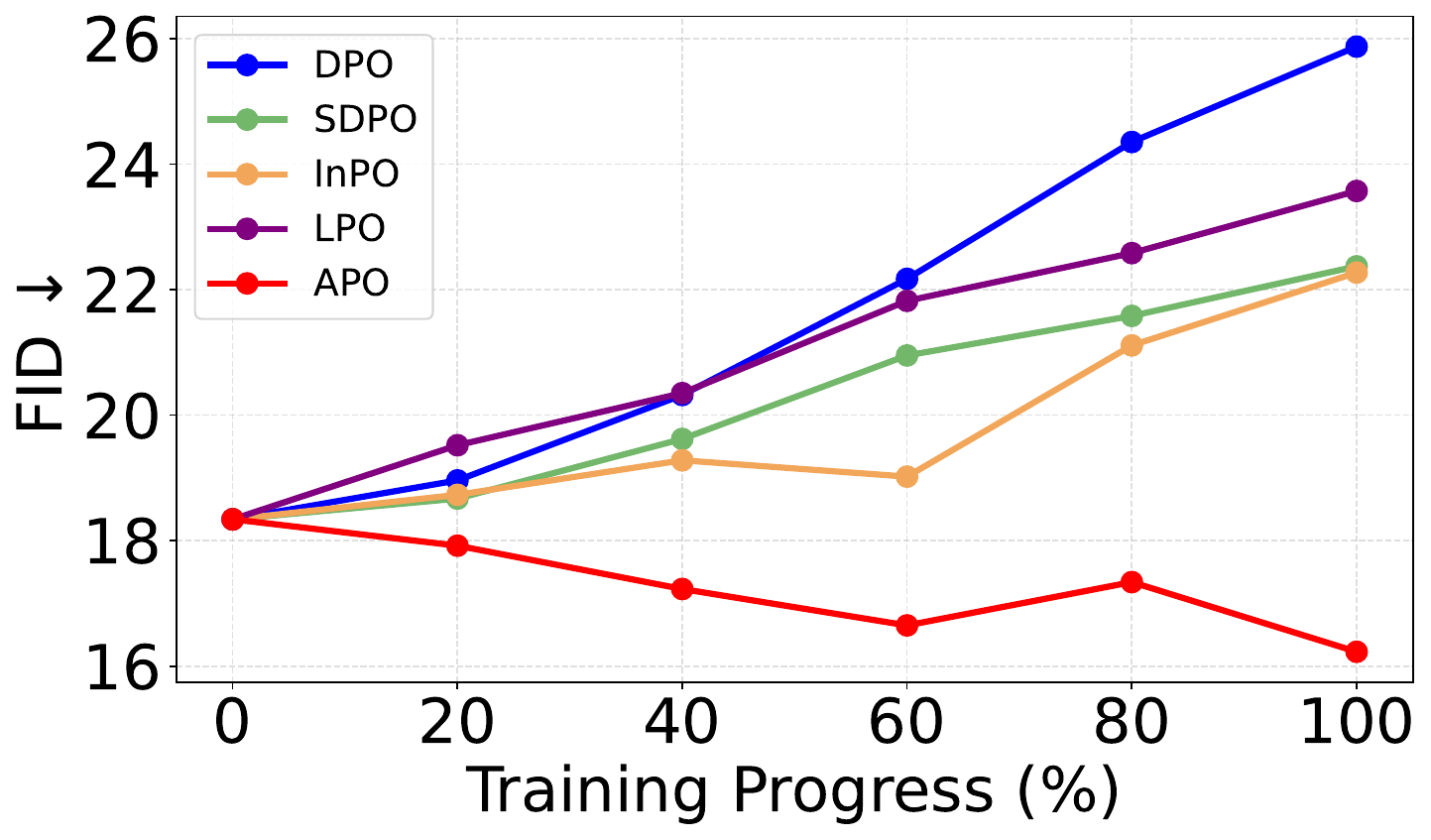}
    \caption{Distributional Stability During Alignment on SDXL.}
    \label{fig:fid_b}
\end{subfigure}
\caption{\textbf{FID-50K throughout preference alignment on COCO-2017.} Standard preference optimization methods exhibit an overall increase in FID as training progresses, whereas APO consistently maintains a lower FID and improves distributional fidelity throughout alignment. On SD1.5 and SDXL, APO reaches FID values of approximately 16.65 and 16.23, respectively, demonstrating improved distributional stability during preference optimization.}
\label{fig:fid}
\end{figure}

\begin{figure}[t]
\centering
\includegraphics[width=\linewidth]{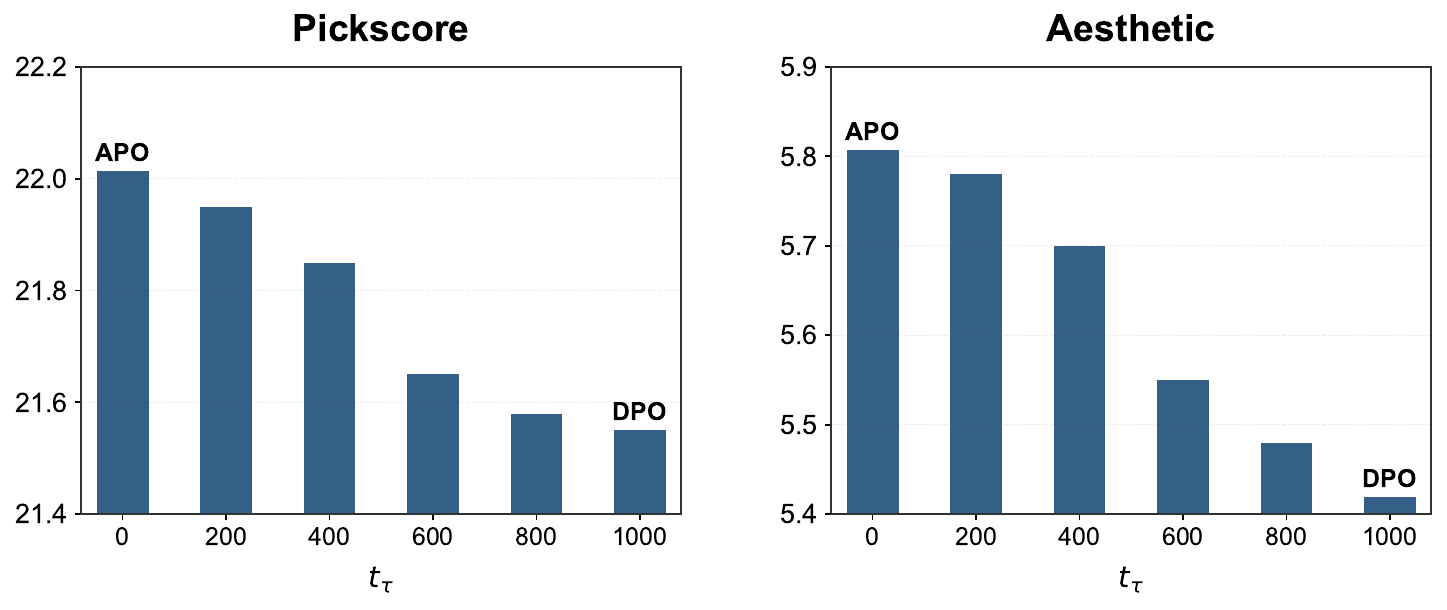}
\caption{\textbf{Ablation study on geometric information.} We compare DPO, APO, and time-threshold variants where APO is applied only for \( t \geq t_\tau \) with \( t_\tau \in \{200, 400, 600, 800\} \). Performance declines monotonically as \( t_\tau \) increases, with the sharpest drop between 400 and 600 steps, indicating that intermediate timesteps benefit most from anisotropic regularization.}
\label{fig:ablation1}
\end{figure}

\paragraph{Qualitative Results}
Figure ~\ref{fig:example} and Figure~\ref{fig:main} presents representative generations from SD1.5 and SDXL under different preference optimization methods. APO consistently produces images with higher visual fidelity, more natural object structures, and better fine-grained details than standard Diffusion-DPO and related baselines. The generated samples also exhibit more coherent global composition. These observations are consistent with our geometric interpretation: by suppressing unstable updates in sensitive directions, APO better preserves the underlying structure learned by the reference diffusion model during alignment.

\paragraph{Training Efficiency}
We report the training efficiency of DPO and APO on SD1.5 and SDXL, including Total Training Hours (h) and peak GPU memory (GB). As shown in \Cref{tab:efficiency-results}, APO reduces total training time to roughly half of standard Diffusion-DPO on both models. \textbf{This speedup mainly comes from faster convergence and fewer training steps. APO reaches stable performance within approximately 900 training steps, whereas standard Diffusion-DPO typically requires around 2000 steps under the same training setup}. This faster convergence is empirically supported by the training dynamics shown in \Cref{fig:dynamics}, where APO's scores ascend sharply within the first 1,000 steps and stabilize smoothly, while DPO still exhibits a gradual climb and requires nearly twice the steps to reach a comparable plateau. Although APO introduces an additional forward evaluation for Jacobian estimation, it stabilizes the alignment process by anchoring optimization in more reliable geometric directions. As a result, APO requires significantly fewer training steps, and benefits from weaker KL regularization without sacrificing stability. In terms of memory, APO only increases peak GPU usage by around \textbf{5\%} compared with DPO. The additional forward pass introduces negligible memory overhead, since no extra backpropagation graph is stored.

\paragraph{Distributional Stability During Alignment}
To further examine the effect of preference optimization on the overall generation distribution, we track the Fr\'echet Inception Distance (FID) throughout training on the COCO-2017 benchmark using the standard FID-50K protocol. Lower FID indicates closer agreement between the generated and reference image distributions. As shown in Figure~\ref{fig:fid}, standard preference optimization methods exhibit a clear tendency toward increasing FID as training progresses. On SD1.5, the FID of DPO, KTO, SDPO, and InPO generally increases throughout training, with DPO showing the largest degradation. In contrast, APO steadily reduces FID from around 20 to approximately \textbf{16.65}. A similar trend is observed on SDXL, where DPO, SDPO, InPO, and LPO all exhibit an overall increase in FID, while APO remains at a substantially lower level throughout training and reaches approximately \textbf{16.23} at the end of alignment.

\subsection{Ablation Studies and Analysis}
\label{sec:ablation}

\subsubsection{Study on Reward Metrics and Image Quality.}
We evaluate APO against three critical baselines on the HPSv2 benchmark: (a) standard DPO; (b) Our APO; (c) a time-threshold variant where APO is applied only on timesteps \( t \geq t_{\tau} \) and standard DPO is used for \( t < t_{\tau} \). We vary \( t_{\tau} \in \{200, 400, 600, 800\} \) and report performance on the HPSv2 benchmark.

As shown in \Cref{fig:ablation1}, the time-threshold results reveal a clear trend: performance decreases monotonically as \( t_{\tau} \) increases (i.e., as APO is restricted to higher noise levels). The most significant drop occurs between \( t_{\tau} = 400 \) and \( t_{\tau} = 600 \), suggesting that the intermediate noise range (approximately 400–600 steps) is where the anisotropic regularization provides the greatest benefit. The decline from \( t_{\tau} = 200 \) to \( t_{\tau} = 400 \) is moderate, while the gaps from \( t_{\tau} = 600 \) to \( 800 \) and from \( t_{\tau} = 800 \) to full DPO are notably smaller.

These observations are consistent with the theoretical analysis in \Cref{sec:prop-comp}. At low noise levels (small \( t \), near the clean image), the data manifold is well-defined and the Jacobian is maximally anisotropic; however, the generation process at this stage mainly refines fine-grained details (textures, fur, edges), which have limited impact on global preference scores. Hence, replacing APO with DPO in the low-noise region (i.e., setting \( t_{\tau} = 200 \)) causes only a modest degradation. In contrast, the intermediate noise range (roughly 400–600 steps) is commonly associated with the formation of semantic layout and coarse structures—where off-manifold drift appears to have the strongest impact on alignment quality. The sharp performance drop when APO is omitted from this region highlights its critical role. At high noise levels (large \( t \)), the Jacobian becomes nearly isotropic, and APO naturally degenerates to DPO; therefore, restricting APO to very high timesteps (\( t_{\tau} = 800 \)) yields diminishing performance gains. Nevertheless, we still recommend applying APO across all timesteps, because even the high-noise region offers a slight but consistent advantage over full DPO.

\subsubsection{Diversity Preservation During Alignment.}
\label{sec:diverity}
We investigate whether APO preserves generation diversity during preference optimization using perceptual and semantic diversity metrics tracked throughout training. \cref{fig:lpips_mpd} (left) reports LPIPS Mean Pairwise Distance (MPD)~\cite{LPIPS}, which measures perceptual diversity among samples generated from the same prompt. Higher values indicate stronger resistance to mode collapse. \cref{fig:dino_variance} (right) shows normalized DINO feature variance~\cite{DINO}, which captures semantic-level diversity.

APO consistently maintains higher diversity than standard DPO across both metrics throughout training, and the gap further widens as training progresses. Standard DPO applies uniform optimization pressure in all directions, causing the distribution to progressively contract toward narrow reward-seeking modes. In contrast, APO preserves greater flexibility along manifold tangent directions while suppressing unstable updates along geometry-sensitive directions. This anisotropic constraint helps maintain greater variation during alignment. 
\begin{figure}[t]
\centering

\begin{subfigure}[b]{0.48\linewidth}
    \centering
    \includegraphics[width=\linewidth]{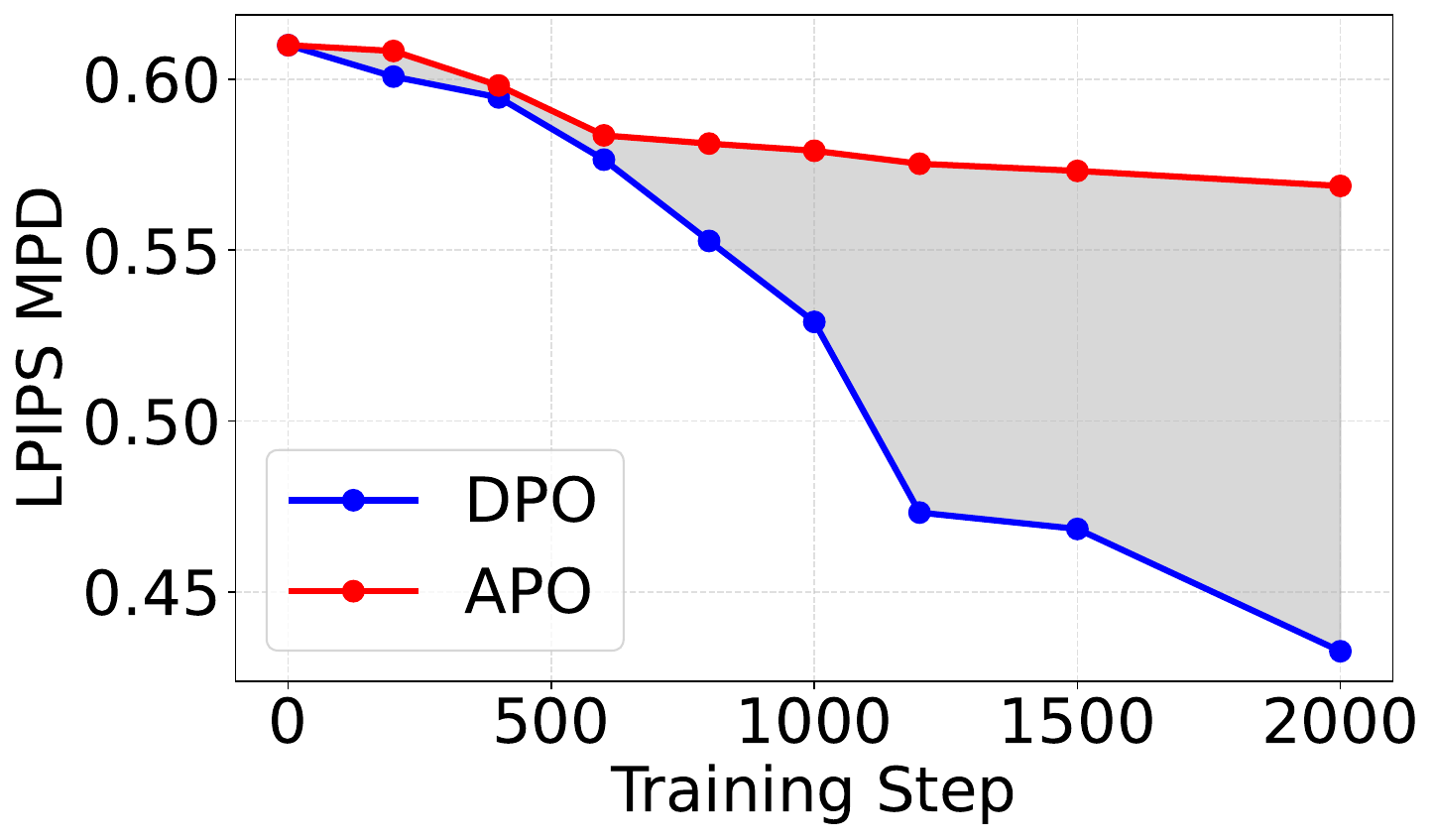}
    \caption{LPIPS Mean Pairwise Distance (MPD).}
    \label{fig:lpips_mpd}
\end{subfigure}
\hfill
\begin{subfigure}[b]{0.48\linewidth}
    \centering
    \includegraphics[width=\linewidth]{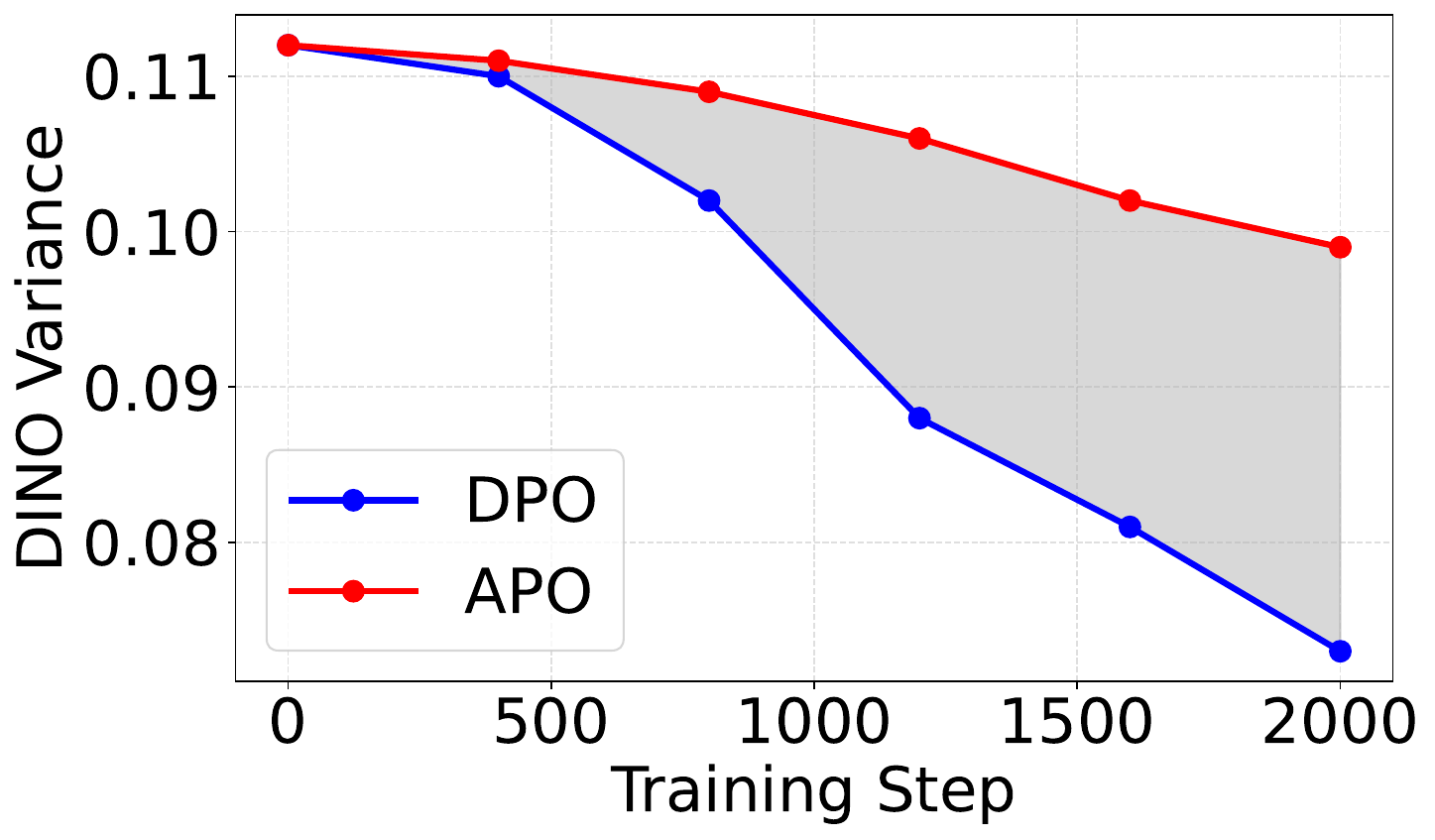}
    \caption{DINO feature variance.}
    \label{fig:dino_variance}
\end{subfigure}

\caption{\textbf{Diversity and semantic variation during training.} The left panel reports LPIPS Mean Pairwise Distance (MPD)~\cite{LPIPS}, while the right panel reports normalized DINO feature variance~\cite{DINO}. APO consistently maintains higher perceptual and semantic diversity than standard DPO throughout training, with the gap becoming more pronounced as training progresses. This suggests that its anisotropic regularization preserves greater flexibility along manifold-tangent directions while suppressing unstable updates along geometry-sensitive directions. Detailed analysis is provided in \Cref{sec:diverity}.}
\label{fig:diversity}
\end{figure}

\begin{table}[t]
\centering
\caption{\textbf{Hyperparameter sensitivity analysis of APO on the geometry strength $\lambda$, DPO temperature $\beta$, finite-difference perturbation radius $r$, and number of probe directions $n_v$.} The best results are \textbf{bolded}, second-best results are \underline{underlined}. The results indicate that APO is robust to moderate changes in its hyperparameters and that the default configuration provides a favorable trade-off between alignment performance and computational cost.}
\label{tab:ablation}

\footnotesize
\renewcommand{\arraystretch}{1.0}
\setlength{\tabcolsep}{4pt}

\begin{subtable}[t]{0.48\linewidth}
\centering
\caption{Ablation results on regularization coefficient $\lambda$}
\label{tab:ablation-lambda}
\begin{tabular}{lcccc}
\toprule
Setting & Aes. $\uparrow$ & Pick. $\uparrow$ & CLIP $\uparrow$ & IR. $\uparrow$ \\
\midrule

$\lambda=0.4$ 
& 5.7215 & 21.427 & 34.769 & 0.6195 \\

$\lambda=0.6$ 
& \underline{5.7821} & \underline{21.914} & \underline{35.353} & \underline{0.6453} \\

$\lambda=0.8$ 
& \textbf{5.8076} & \textbf{22.014} & \textbf{35.873} & \textbf{0.6887} \\

$\lambda=1.0$ 
& 5.7602 & 21.734 & 34.972 & 0.5711 \\

\bottomrule
\end{tabular}
\end{subtable}
\hfill
\begin{subtable}[t]{0.48\linewidth}
\centering
\caption{Ablation results on DPO temperature $\beta$}
\label{tab:ablation-beta}
\begin{tabular}{lcccc}
\toprule
Setting & Aes. $\uparrow$ & Pick. $\uparrow$ & CLIP $\uparrow$ & IR. $\uparrow$ \\
\midrule

$\beta=300$ 
& 5.7717 & 21.836 & \underline{35.509} & 0.6128 \\

$\beta=500$ 
& \textbf{5.8076} & \textbf{22.014} & \textbf{35.873} & \textbf{0.6887}  \\

$\beta=1000$ 
& \underline{5.7930} & \underline{21.875} & 35.412 & \underline{0.6695} \\

$\beta=1500$ 
& 5.7534 & 22.012 & 35.869 & 0.6358 \\

\bottomrule
\end{tabular}
\end{subtable}

\vspace{2pt}

\begin{subtable}[t]{0.48\linewidth}
\centering
\caption{Ablation results on perturbation radius $r$}
\label{tab:ablation-r}
\begin{tabular}{lcccc}
\toprule
Setting & Aes. $\uparrow$ & Pick. $\uparrow$ & CLIP $\uparrow$ & IR. $\uparrow$ \\
\midrule

$r=10^{-4}$ 
& 5.8013 & \underline{21.998} & \underline{35.861} & \underline{0.6871} \\

$r=10^{-3}$ 
& \textbf{5.8076} & \textbf{22.014} & \textbf{35.873} & \textbf{0.6887}  \\

$r=10^{-2}$ 
& \underline{5.8032} & 21.971 & 35.812 & 0.6830 \\

$r=10^{-1}$ 
& 5.7534 & 21.853 & 35.604 & 0.6612 \\

\bottomrule
\end{tabular}
\end{subtable}
\hfill
\begin{subtable}[t]{0.48\linewidth}
\centering
\caption{Ablation results on number of probe directions $n_v$}
\label{tab:ablation-nv}
\begin{tabular}{lcccc}
\toprule
Setting & Aes. $\uparrow$ & Pick. $\uparrow$ & CLIP $\uparrow$ & IR. $\uparrow$ \\
\midrule

$n_v=1$ 
& 5.8076 & 22.014 & 35.873 & 0.6887  \\

$n_v=2$ 
& 5.8083 & 22.016 & 35.877 & 0.6891 \\

$n_v=3$ 
& \underline{5.8092} & \underline{22.020} & \underline{35.881} & \underline{0.6893} \\

$n_v=4$ 
& \textbf{5.8104} & \textbf{22.027} & \textbf{35.884} & \textbf{0.6896} \\

\bottomrule
\end{tabular}
\end{subtable}

\end{table}

\subsubsection{Hyperparameter Sensitivity Analysis}
\label{sec:ablation_sensitivity}
\Cref{tab:ablation} reports the sensitivity of APO to its key hyperparameters: geometry strength $\lambda$, DPO regularization coefficient $\beta$, finite-difference perturbation radius $r$, and number of probe directions $n_v$.

\paragraph{Anisotropic regularization coefficient $\lambda$} Setting $\lambda = 0$ recovers standard Diffusion-DPO and yields the lowest scores across all metrics, confirming that the anisotropic correction is essential rather than merely additive regularization. Once the geometric term is added ($\lambda > 0$), performance improves immediately. We adopt $\lambda = 0.8$ as the default.

\paragraph{DPO temperature $\beta$} We observe that APO achieves its best performance with a smaller $\beta$ than standard Diffusion-DPO and related implicit-reward alignment methods. Empirically, introducing the anisotropic geometric term already reshapes the optimization landscape by reweighting prediction discrepancies according to the local geometry induced by the reference Jacobian. As a result, strong global preference scaling becomes less necessary for achieving effective alignment. This observation is also consistent with the improved training efficiency reported in \cref{sec:main}, where APO converges with substantially fewer optimization steps and faster alignment dynamics than standard Diffusion-DPO.

\paragraph{Perturbation radius $r$} APO is not sensitive to the numerical value of the perturbation radius; comparable performance is achieved for $r$ between $10^{-4}$ and $10^{-2}$, with an optimum at $r = 10^{-3}$. Larger values ($r = 10^{-1}$) cause the finite-difference approximation to exit the linear regime of $\epsilon_{\text{ref}}$, introducing bias. Smaller values ($r = 10^{-4}$) show degraded performance, possibly due to numerical instability.

\paragraph{Number of probe directions $n_v$} \label{sec:nv}A single probe direction ($n_v = 1$) already achieves strong results, and increasing $n_v$ yields only marginal gains. This is consistent with the analysis in \Cref{sec:prop-comp}: the diffusion trajectory itself already induces strong temporal averaging of the anisotropic direction estimates, making single-probe estimation sufficiently stable in practice.

\section{Conclusion and Limitation}
\paragraph{Conclusion} In this paper, we presented APO, a geometry-aware preference optimization framework for diffusion models. We revisit diffusion preference alignment from the perspective of manifold geometry, and identify off-manifold drift as a fundamental failure mode induced by isotropic optimization in high-dimensional denoising space. To address this issue, APO constructs a Jacobian-induced anisotropic metric from the reference score field, selectively suppressing geometry-sensitive directions while preserving flexible exploration along manifold-aligned directions. Our results suggest that incorporating intrinsic geometric structure into preference optimization provides an effective and principled direction for robust diffusion alignment.

\paragraph{Limitation} APO assumes that the reference score field provides a meaningful local approximation of the underlying data manifold. When the preference data distribution itself contains substantial off-manifold corruption or systematically noisy annotations, the induced geometry may become less reliable and potentially bias the alignment process. This may limit the effectiveness of APO in heavily distribution-shifted or weakly curated preference-learning scenarios. Extending APO toward more robust geometric estimation under noisy preference supervision remains an important direction for future work.

\clearpage

\begin{table}[t]
\centering
\caption{\textbf{Win Rate comparison on SDXL.} Results greater than 50\% indicate that APO is preferred over the corresponding baseline and are \textbf{bolded}. APO consistently achieves above-50\% win rates against most baselines, demonstrating its robust preference alignment performance across different benchmarks and evaluation criteria.}
\label{tab:combined-results-sdxl}
\footnotesize
\setlength{\tabcolsep}{10pt}
\begin{tabular}{@{\hspace{4pt}}l lccccc@{\hspace{4pt}}}
\toprule
\specialrule{0.05em}{0pt}{2pt}
& & \multicolumn{5}{c}{Win Rate (\%) vs. Baseline $\uparrow$} \\
\cmidrule{3-7}
Dataset & Method & Aes. & Pick. & CLIP & IR. & Avg. \\
\midrule
\multirow{7}{*}{HPSv2}
& vs. Base (SDXL)   & \textbf{89.28} & \textbf{90.12} & \textbf{58.12} & \textbf{87.13} & \textbf{81.16}\\
& vs. SFT           & \textbf{56.85} & \textbf{60.17} & \textbf{57.32} & \textbf{58.29} & \textbf{58.16}\\
& vs. Diffusion-DPO & \textbf{60.16} & \textbf{65.37} & \textbf{52.61} & \textbf{63.23} & \textbf{60.34}\\
& vs. SDPO          & \textbf{59.14} & \textbf{62.33} & \textbf{54.10} & \textbf{57.58} & \textbf{58.29}\\
& vs. InPO          & \textbf{54.06} & \textbf{59.37} & 47.80 & \textbf{56.02} & \textbf{54.66}\\
& vs. LPO           & \textbf{54.32} & 46.34 & \textbf{52.92} & \textbf{52.83} & \textbf{51.60}\\
\midrule
\multirow{7}{*}{Parti-Prompts}
& vs. Base (SDXL)   & \textbf{88.46} & \textbf{89.30} & \textbf{57.43} & \textbf{90.16} & \textbf{81.34}\\
& vs. SFT           & \textbf{57.24} & \textbf{58.56} & \textbf{53.01} & \textbf{60.24} & \textbf{57.26}\\
& vs. Diffusion-DPO & \textbf{62.39} & \textbf{57.04} & 46.60 & \textbf{59.30} & \textbf{56.33}\\
& vs. SDPO          & \textbf{55.28} & \textbf{56.37} & \textbf{52.93} & \textbf{54.31} & \textbf{54.72}\\
& vs. InPO          & \textbf{54.27} & \textbf{54.48} & \textbf{51.02} & \textbf{53.27} & \textbf{53.26}\\
& vs. LPO           & \textbf{53.46} & \textbf{50.42} & \textbf{54.35} & \textbf{52.15} & \textbf{52.60}\\
\midrule
\multirow{7}{*}{Pick-a-Pic}
& vs. Base (SDXL)   & \textbf{87.62} & \textbf{89.93} & \textbf{53.19} & \textbf{88.02} & \textbf{79.69}\\
& vs. SFT           & \textbf{60.13} & \textbf{69.27} & \textbf{60.14} & \textbf{58.30} & \textbf{61.96}\\
& vs. Diffusion-DPO & \textbf{64.44} & \textbf{61.07} & \textbf{52.38} & \textbf{56.34} & \textbf{58.56}\\
& vs. SDPO          & \textbf{61.37} & \textbf{54.36} & \textbf{54.57} & \textbf{55.61} & \textbf{56.48}\\
& vs. InPO          & \textbf{57.27} & \textbf{55.21} & 46.37 & \textbf{52.96} & \textbf{52.95}\\
& vs. LPO           & \textbf{54.32} & \textbf{54.26} & \textbf{54.22} & \textbf{51.85} & \textbf{53.66}\\
\bottomrule
\end{tabular}
\end{table}

\begin{figure*}[h]
  \centering
  \includegraphics[height=0.45\textheight, keepaspectratio]{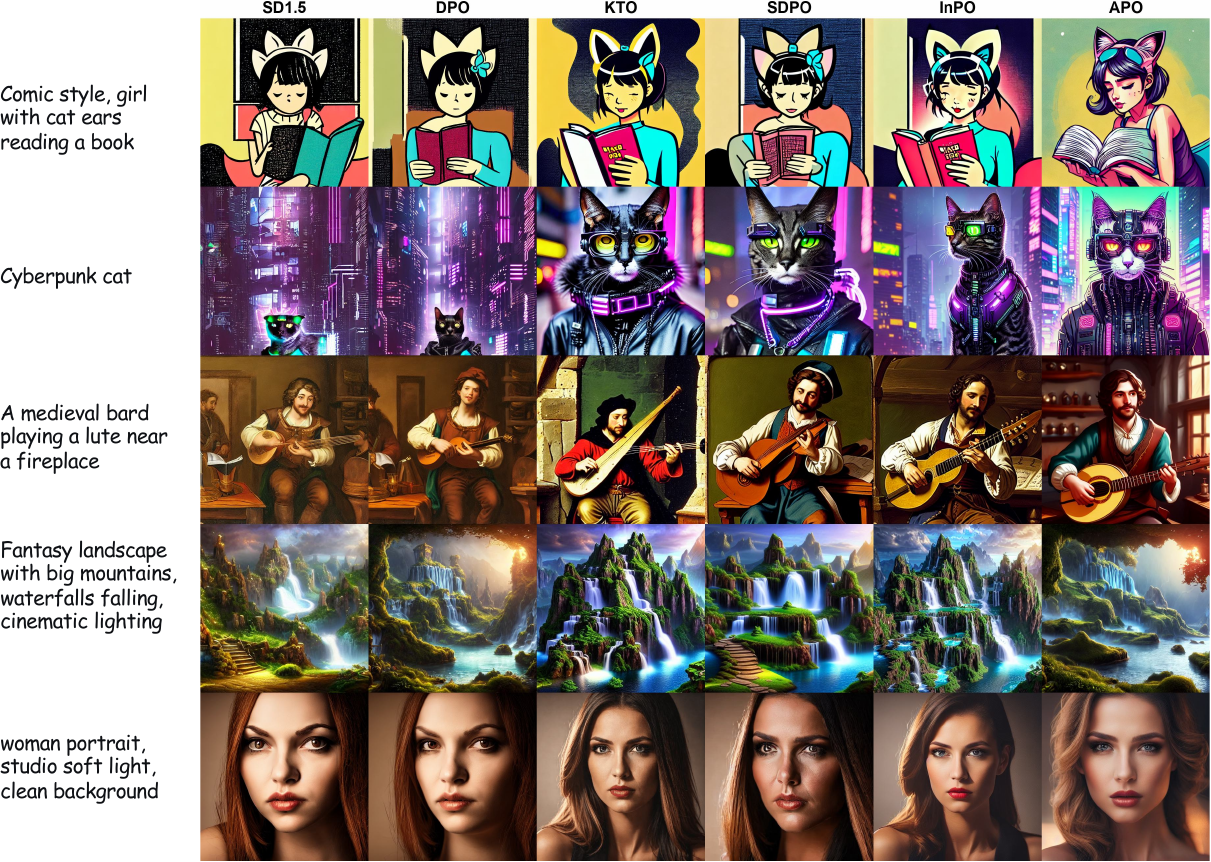}
  \hspace{2.0cm}
  \hspace*{-1.0cm} 
  \includegraphics[height=0.45\textheight, keepaspectratio]{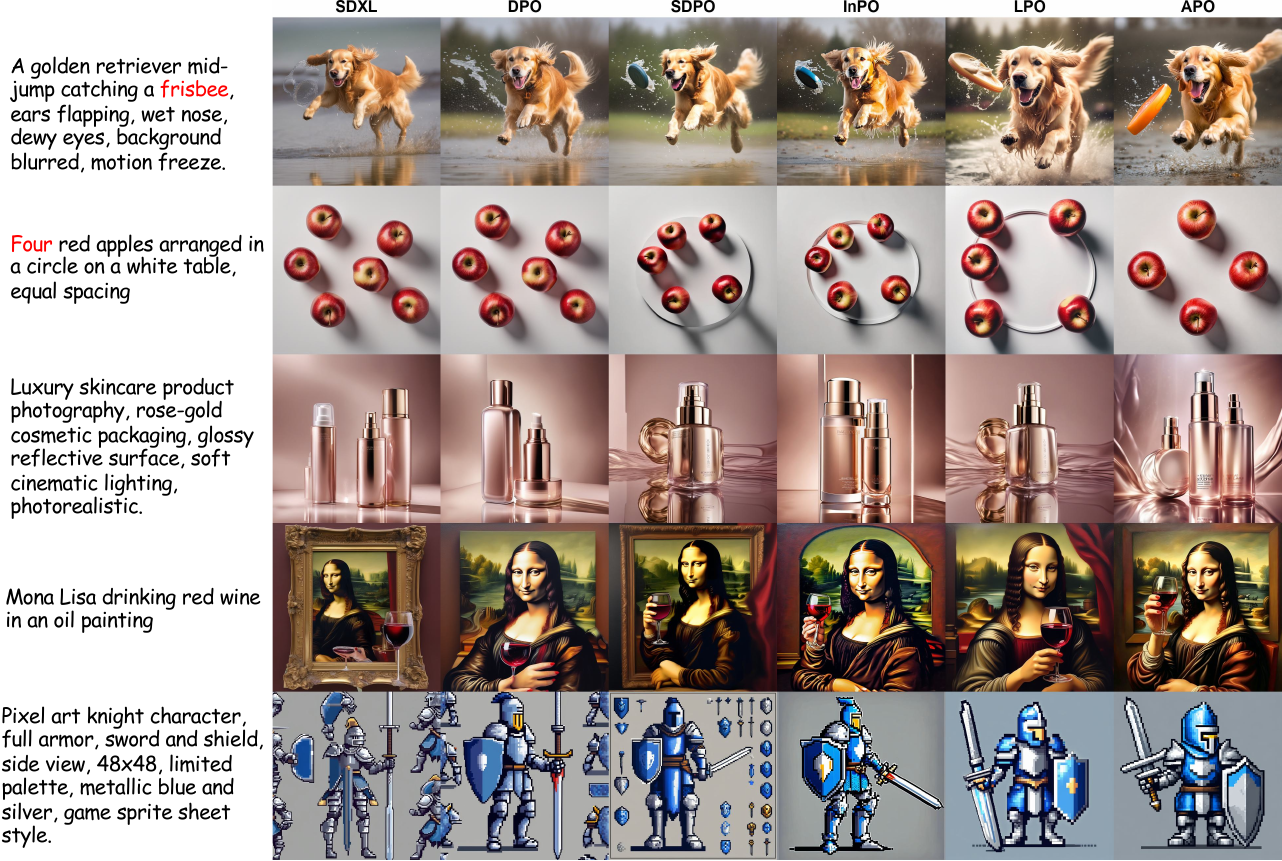}
  \caption{\textbf{Qualitative comparisons among baselines on SD1.5 and SDXL.} APO achieves higher prompt alignment, higher aesthetic quality, and richer details; most importantly, the realism of the images is greatly improved, indicating that the manifold constraint is taking effect.}
  \label{fig:main}
\end{figure*}

\clearpage
\begin{tcolorbox}[boxsep=1pt,top=2pt,bottom=2pt,left=4pt,right=4pt]
\textit{
\textbf{GenAI Usage Disclosure.}
During the preparation of this work, the authors used generative AI tools (including ChatGPT) solely for English polishing, grammar correction, and improving writing clarity. The scientific ideas, methodology, theoretical analysis, experiments, and conclusions were developed and verified entirely by the authors. The authors take full responsibility for the content of this paper.}
\end{tcolorbox}

\appendix

\bibliographystyle{elsarticle-num}
\bibliography{references}
\end{document}